\documentclass[final,12pt]{arxiv}

\usepackage[utf8]{inputenc} 
\usepackage[T1]{fontenc}    
\usepackage{hyperref}       
\usepackage{url}            
\usepackage{booktabs}       
\usepackage{amsfonts}       
\usepackage{nicefrac}       
\usepackage{microtype}      
\usepackage{xcolor}         

\usepackage{dsfont}
\usepackage{enumitem}
\usepackage[toc,page,header]{appendix}
\usepackage{lipsum,wrapfig}

\usepackage[english]{babel}
\usepackage{amsmath, amssymb}
\usepackage{graphicx}
\usepackage{latexsym}
\usepackage{graphicx}
\usepackage{cases}
\usepackage[mathscr]{euscript}
\usepackage{xcolor}
\usepackage{colortbl}
\usepackage{thmtools}
\usepackage{thm-restate}

\usepackage{mathtools}
\usepackage{subcaption}

\usepackage{multirow}
\usepackage{wrapfig}

\newtheorem{lemma}[theorem]{Lemma}

\usepackage{bm}

\let\vec\undefined
\usepackage{MnSymbol}
\DeclareMathAlphabet\mathbb{U}{msb}{m}{n}
\usepackage{xpatch}

\def\Rset{\mathbb{R}}

\DeclareMathOperator*{\E}{\mathbb E}
\DeclareMathOperator*{\argmax}{argmax}

\DeclarePairedDelimiter{\abs}{\lvert}{\rvert} 
\DeclarePairedDelimiter{\bracket}{[}{]}
\DeclarePairedDelimiter{\curl}{\{}{\}}

\DeclarePairedDelimiter{\paren}{(}{)}

\newcommand{\sD}{{\mathscr D}}
\newcommand{\sE}{{\mathscr E}}

\newcommand{\sH}{{\mathscr H}}

\newcommand{\sM}{{\mathscr M}}

\newcommand{\sR}{{\mathscr R}}

\newcommand{\sX}{{\mathscr X}}
\newcommand{\sY}{{\mathscr Y}}

\newcommand{\sfL}{{\mathsf L}}

\newcommand{\ov}{\overline}

\newcommand{\e}{\epsilon}

\newcommand{\ignore}[1]{}

\newcommand{\ul}{\ov l}
\newcommand{\uc}{\ov c}

\DeclareMathOperator{\sign}{sign}
\def\Nset{\mathbb{N}}

\newcommand{\num}{{n_e}}

\newcommand{\rr}{{\sf r}}
\newcommand{\expert}{{g}}
\newcommand{\ldef}{L_{\rm{def}}}

\usepackage{pifont}
\newcommand{\cmark}{\ding{51}}%
\newcommand{\xmark}{\ding{55}}%

\usepackage{multirow}

\hypersetup{
  breaklinks   = true, 
  colorlinks   = true, 
  urlcolor     = blue, 
  linkcolor    = blue, 
  citecolor   = blue 
}

\usepackage[toc, page, header]{appendix}
\usepackage{times}
\title[Regression with Multi-Expert Deferral]
      {Regression with Multi-Expert Deferral}

\arxivauthor{%
 \Name{Anqi Mao} \Email{aqmao@cims.nyu.edu}\\
 \addr Courant Institute of Mathematical Sciences, New York%
 \AND
 \Name{Mehryar Mohri} \Email{mohri@google.com}\\
 \addr Google Research and Courant Institute of Mathematical Sciences, New York%
 \AND
 \Name{Yutao Zhong} \Email{yutao@cims.nyu.edu}\\
 \addr Courant Institute of Mathematical Sciences, New York%
}

\begin{document}

\maketitle

\begin{abstract}
Learning to defer with multiple experts is a framework where the
learner can choose to defer the prediction to several experts. While
this problem has received significant attention in classification
contexts, it presents unique challenges in regression due to the
infinite and continuous nature of the label space.  In this work, we
introduce a novel framework of \emph{regression with deferral}, which
involves deferring the prediction to multiple experts.  We present a
comprehensive analysis for both the single-stage scenario, where there
is simultaneous learning of predictor and deferral functions, and the
two-stage scenario, which involves a pre-trained predictor with a
learned deferral function. We introduce new surrogate loss functions
for both scenarios and prove that they are supported by
$\sH$-consistency bounds. These bounds provide consistency guarantees
that are stronger than Bayes consistency, as they are non-asymptotic
and hypothesis set-specific. Our framework is versatile, applying to
multiple experts, accommodating any bounded regression losses,
addressing both instance-dependent and label-dependent costs, and
supporting both single-stage and two-stage methods. A by-product is
that our single-stage formulation includes the recent \emph{regression
with abstention} framework \citep{cheng2023regression} as a special
case, where only a single expert, the squared loss and a
label-independent cost are considered. Minimizing our proposed loss
functions directly leads to novel algorithms for regression with
deferral. We report the results of extensive experiments showing the
effectiveness of our proposed algorithms.
\end{abstract}



\section{Introduction}

The accuracy of learning algorithms can be greatly enhanced by
redirecting uncertain predictions to experts or advanced pre-trained
models. Experts can be individuals with specialized domain knowledge
or more sophisticated, albeit costly, pre-trained models. The cost of
an expert is important to consider, as it may capture the
computational resources it requires or the quality of its performance.
The cost can further be instance-dependent and label-dependent.

How can we effectively assign each input instance to the most suitable
expert among a pool of several, considering both accuracy and cost?
This is the challenge of \emph{learning to defer in the presence of
multiple experts}, which is prevalent in various domains, including
natural language generation tasks, notably large language models
(LLMs) \citep{WeiEtAl2022, bubeck2023sparks}, speech recognition,
image annotation and classification, medical diagnosis, financial
forecasting, natural language processing, computer vision, and many
others.

This paper deals with the problem of learning to defer with multiple
experts in the regression setting.  While this problem has received
significant attention in classification contexts
\citep{madras2018predict,hemmer2022forming, keswani2021towards, kerrigan2021combining,
  straitouri2022provably,
  benz2022counterfactual,verma2023learning,MaoMohriZhong2023,
  MaoMohriZhong2024,tailor2024learning}, it presents unique challenges in regression due
to the infinite and continuous nature of the label space. In
particular, the \emph{score-based formulation} commonly used in
classification is inapplicable here, since regression problems cannot
be represented using multi-class scoring functions, with auxiliary
labels corresponding to each expert.

Our approach involves defining prediction and deferral functions,
consistent with previous studies in classification
\citep{MaoMohriZhong2023, MaoMohriZhong2024}. We present a
comprehensive analysis for both the single-stage scenario
(simultaneous learning of predictor and deferral functions)
(Section~\ref{sec:single-stage}), and the two-stage scenario
(pre-trained predictor with learned deferral function)
(Section~\ref{sec:two-stage}). We introduce new surrogate loss
functions for both scenarios and prove that they are supported by
$\sH$-consistency bounds. These are consistency guarantees that are
stronger than Bayes consistency, as they are non-asymptotic and
hypothesis set-specific. Our framework is versatile, applying to
multiple experts, accommodating any bounded regression losses,
addressing both instance-dependent and label-dependent costs, and
supporting both single-stage and two-stage methods. We also
instantiate our formulations in the special case of a single expert
(Section~\ref{sec:single-expert}), and demonstrate that our
single-stage formulation includes the recent \emph{regression with
abstention} framework \citep{cheng2023regression} as a special case,
where only a single expert, the squared loss and a label-independent
cost are considered. In Section~\ref{sec:experiments}, we report the
results of extensive experiments showing the effectiveness of our
proposed algorithms.

\textbf{Previous related work.}
The problem of learning to defer, or the special case of learning with
abstention characterized by a single expert and constant cost, has 
received much attention in classification tasks. Previous work on
this topic mainly includes the following formulations or methods:
\emph{confidence-based}, \emph{predictor-rejector}, \emph{score-based}, 
and \emph{selective classification}.

In the \emph{confidence-based formulation}, the rejection function $r$ is based on
the magnitude of the value of the predictor $h$
\citep{Chow1957,chow1970optimum,bartlett2008classification,
  yuan2010classification,WegkampYuan2011}. This approach has been
further extended to multi-class classification in
\citep{ramaswamy2018consistent,NiCHS19}, where the function $r$ is
based on the magnitude of the value of the probability (e.g., softmax)
corresponding to the predictor $h$.
This formulation becomes inapplicable in regression, since in this setting the prediction value cannot be interpreted as a measure of confidence.

The \emph{score-based formulation} was proposed in the multi-class
classification scenario, where the multi-class categories are
augmented with additional labels corresponding to the experts, and the
deferral is determined using the highest score
\citep{mozannar2020consistent,verma2022calibrated,caogeneralizing,
  MaoMohriZhong2024score,verma2023learning,MaoMohriZhong2024}. However,
this formulation is also inapplicable in regression, since regression
problems cannot be represented using multi-class scoring functions
with auxiliary labels corresponding to each expert.

The approach of learning based on two distinct yet jointly learned
functions $h$ and $r$ in this paper is commonly referred as the
\emph{predictor-rejector formulation}
\citep{cortes2016learning,CortesDeSalvoMohri2016bis,
  charoenphakdee2021classification,CortesDeSalvoMohri2023,
  MohriAndorChoiCollinsMaoZhong2024learning,MaoMohriZhong2024predictor}. We
show that this method can be extended to the regression setting for
deferral with multiple experts, which underscores its versatility and significance.

An alternative approach of \emph{selective classification}
\citep{el2010foundations,wiener2011agnostic,el2012active,
wiener2012pointwise,wiener2015agnostic,geifman2017selective,
geifman2019selectivenet} optimizes non-abstained sample
generalization error under a fixed selection rate. However, this
method does not apply to the deferral case where the cost depends on
the label $y$ and where there are multiple experts.  Moreover, it has been reported to perform suboptimally compared to the
predictor-rejector formulation in regression with abstention settings
with constant cost and a single expert \citep{cheng2023regression}.

More recently, a series of publications \citep{MaoMohriZhong2023,
  MohriAndorChoiCollinsMaoZhong2024learning,MaoMohriZhong2024predictor}
have explored the two-stage method of learning with
deferral or abstention, wherein the predictor $h$ is first learned and
subsequently used in the learning process of the deferral function
$r$. This scenario is crucial in practice because the predictor is
often given and often cannot be retrained. This method
also differs from post-hoc approaches \citep{okati2021differentiable,
  narasimhanpost}, which are not applicable to existing predictors
trained in the standard classification scenario. In this work, we will
study both the single-stage and two-stage methods for regression with
deferral.

In the special case of regression with abstention (corresponding to a
single expert and label-independent cost case),
\citet{wiener2012pointwise} characterized the optimal selector for
selective regression, \citet{zaoui2020regression} studied
non-parametric algorithms,
\citet{geifman2019selectivenet} and \citet{jiang2020risk} explored the selective
classification using neural network-based algorithms;
\citet{shah2022selective} used the selective classification with greedy
algorithms; \citet{de2020regression} proposed an approximate procedure for learning a linear hypothesis that determines which training instances should be deferred. They then used a nearest neighbor approach to defer on new instances;
  and \citet{li2023no} investigated a two-step no-rejection
learning strategy. However, none of these previous publications studied surrogate losses
for regression with abstention. This excludes
\citep{cheng2023regression}, who proposed a single-stage surrogate
loss for learning the predictor-rejector pair. We will show that their
method coincides wuth a special case of our single-stage regression with deferral surrogate losses, where there is a single expert and where the cost does not depend on the label $y$.

Another line of work has explored dynamic classifier selection or dynamic ensemble selection \citep{ko2008dynamic,cruz2018dynamic,ekanayake2023sequential}, which aims to select the most competent classifiers or an ensemble of classifiers in the local region where each instance is located. While these methods also consider how to select an expert from a pool of several, their primary mechanism involves dividing the feature space into distinct regions. In contrast, learning to defer with multiple experts aims to learn a deferral function by minimizing a surrogate loss that accounts for the accuracy and cost of each expert across all instances.

Learning to defer with multiple experts in the classification
setting has been studied in
\citep{MaoMohriZhong2023,verma2023learning,
  MaoMohriZhong2024}. \citet{verma2023learning} and \cite{MaoMohriZhong2024}
investigated the single-stage scenario with a score-based formulation,
while \citet{MaoMohriZhong2023} explored the two-stage scenario with
both score-based and predictor-rejector formulations. However, as previously highlighted, the score-based formulation does not apply in the
regression setting. Our new predictor-rejector
formulation not only overcomes this limitation, but also provides the
foundation for the design of new deferral algorithms for
classification.

\section{Preliminaries}

\textbf{Learning scenario of regression.} We first describe the
familiar problem of supervised regression and introduce our
notation. Let $\sX$ be the input space and $\sY \subseteq \Rset$ the
label space. We write $\sD$ to denote a distribution over $\sX \times
\sY$. Let $\sH_{\rm{all}}$ be the family of all real-valued measurable
functions $h \colon \sX \to \sY$, and let $\sH \subseteq
\sH_{\rm{all}}$ be the hypothesis set adopted. The learning challenge
in regression is to use the labeled sample to find a hypothesis $h \in
\sH$ with small expected loss or generalization error $\sE_{\sfL}(h)$,
with $\sE_{\sfL}(h) = \E_{(x, y) \sim \sD}\bracket*{\sfL (h(x), y)}$,
where $\sfL \colon \sY \times \sY \to \Rset_{+}$ is a loss function
used to measure the magnitude of error in the regression. In the most
common case, where $\sfL$ is the squared loss $\sfL_2$ defined by
$\sfL_2(y', y) = \abs*{y' - y}^2$, this represents the mean squared
error. In the case where $\sfL$ is the $\sfL_1$ loss defined by
$\sfL_1(y', y) = \abs*{y' - y}$, this represents the mean absolute
error. More generally, $\sfL$ can be an $\sfL_p$ loss, defined by
$\sfL_p(y', y) = \abs*{y' - y}^p$ for all $y', y \in \sY$, for some $p
\geq 1$. In this work, we will consider an arbitrary regression loss
function $\sfL$, subject to the boundedness assumption, that is
$\sfL(y', y) \leq \ul$ for some constant $\ul > 0$ and for all $y, y'
\in \sY$. This assumption is commonly adopted in the theoretical
analysis of regression \citep{mohri2018foundations}.

\textbf{Regression with deferral.}
We introduce a novel framework where a learner can defer predictions to multiple experts, $\expert_1, \ldots, \expert_{n_e}$. Each expert may represent a pre-trained model or a human expert.
The learner's output is a pair $(h, r)$, where $h \colon \sX \to \sY$
is a prediction function and $r \colon \sX \times \curl*{0, 1, \ldots,
  \num} \to \Rset$ a deferral function.
For any input $x$, $\rr(x) = \argmax_{y \in [\num]} r(x, y) = j$ is
the expert deferred to when $j > 0$, no deferral if $j = 0$.  The
learner makes the prediction $h(x)$ when $\rr(x) = 0$, or defers to
$\expert_j$ when $\rr(x) = j > 0$.  Deferral incurs the cost
$\sfL(\expert_{j}(x), y) + \alpha_j$, where $\alpha_j$ is a base
cost. Non-deferral incurs the cost $\sfL(h(x), y)$.

Let $\sH_{\rm{all}}$ and $\sR_{\rm{all}}$ denote the family of all
measurable functions $h\colon \sX \to \sY$ and $r\colon \sX \times
\curl*{0, 1, \ldots, \num} \to \Rset$ respectively.  Given a
hypothesis set $\sH \subset \sH_{\rm{all}}$ and a hypothesis set $\sR
\subset \sR_{\rm{all}}$, the goal of the regression with deferral
problem consists of using the labeled sample to find a pair $(h, r)
\in (\sH, \sR)$ with small expected deferral loss
$\sE_{\ldef}(h, r) = \E_{(x, y) \sim
  \sD}\bracket*{\ldef(h, r, x, y)}$, where $\ldef$ is
defined for any $(h, r)\in
\sH \times \sR$ and $(x, y)\in \sX \times \sY$ by
\begin{equation}
\label{eq:def}
\mspace{-2mu} \ldef(h, r, x, y)
\!=\! \sfL(h(x), y) 1_{\rr(x) = 0}
+ \mspace{-4mu} \sum_{j = 1}^{\num} c_j(x,y) 1_{\rr(x) = j}
\mspace{-10mu}
\end{equation}
and $c_j(x, y) > 0$ is a cost function, which can be typically chosen
as $\alpha_j + \sfL\paren*{\expert_{j}(x), y}$ for an expert
$\expert_{j}$ and a base cost $\alpha_j> 0$ as mentioned before.
Here, we adopt a general cost functions $c_j$ for any $j$, and only
require that the cost remains bounded: $c_j(x, y) \leq \uc_j$ for all
$(x, y) \in \sX \times \sY$, for some constant $\uc_j > 0$.

\textbf{Learning with surrogate losses.} As with most target losses
in learning problems, such as the zero-one loss in classification
\citep{Zhang2003,bartlett2006convexity,zhang2004statistical,
  tewari2007consistency} and the classification with abstention loss
\citep{bartlett2008classification,cortes2016learning}, directly
minimizing the deferral loss $\ldef$ is computationally hard
for most hypothesis sets due to its non-continuity and
non-differentiability. Instead, surrogate losses are proposed and
adopted in practice. Examples include the hinge loss in binary
classification \citep{cortes1995support}, the (multinomial) logistic
loss in multi-class classification
\citep{Verhulst1838,Verhulst1845,Berkson1944,Berkson1951}, and the
predictor-rejector abstention loss in classification with abstention
\citep{cortes2016learning}. We will derive surrogate losses for the
deferral loss.

Given a surrogate loss $L \colon (h, r, x, y) \mapsto \Rset_{+}$, we
denote by $\sE_{L}(h, r)$ the generalization error of a pair $(h, r)$,
defined as
\begin{equation*}
\sE_{L}(h, r) = \E_{(x, y) \sim \sD} \bracket*{L(h, r, x, y)}.
\end{equation*}
Let $\sE_{L}(\sH, \sR) = \inf_{h\in \sH, r\in \sR} \sE_{L}(h, r)$ be
the best-in-class error within the family $\sH \times \sR$. One
desired property for surrogate losses in this context is
\emph{Bayes-consistency} \citep{steinwart2007compare}. This means that
minimizing the expected surrogate loss over the family of all
measurable functions leads to minimizing the expected deferral loss
over the same family. More precisely, for a surrogate loss $L \colon
(h, r, x, y) \mapsto \Rset_{+}$, it is \emph{Bayes-consistent} with
respect to $\ldef$ if,
\begin{align*}
\sE_{L}(h_n, r_n) - \sE_{L}(\sH, \sR) \xrightarrow{n \rightarrow \plus \infty} 0
\implies \sE_{\ldef}(h_n, r_n) - \sE_{ \ldef}(\sH, \sR) \xrightarrow{n \rightarrow \plus \infty} 0
\end{align*}
for all sequences $\curl*{(h_n, r_n)}_{n \in \Nset} \subset \sH \times
\sR$ and all distributions. Recently,
\citet{awasthi2022h,awasthi2022multi} (see also
\citep{MaoMohriZhong2023ranking,MaoMohriZhong2023rankingabs,mao2023cross,zheng2023revisiting,MaoMohriZhong2023characterization,MaoMohriZhong2023structured}) pointed out that
Bayes-consistency does not take into account the hypothesis set $\sH$
and is non-asymptotic. Thus, they proposed a stronger guarantee called
\emph{$\sH$-consistency bounds}. In our context, a surrogate loss $L$
is said to admit an $(\sH, \sR)$-consistency bound with respect to
$\ldef$ if, for all $(h, r) \in \sH \times \sR$ and all
distributions, the following inequality holds:
\begin{equation*}
f \paren*{\sE_{\ldef}(h, r) - \sE^*_{\ldef}(\sH, \sR)} \leq \sE_{L}(h, r) - \sE^*_{L}(\sH, \sR)
\end{equation*}
for some non-decreasing function $f \colon \sR_{+} \to \sR_{+}$. In particular, when
$(\sH, \sR) = \paren*{\sH_{\rm{all}}, \sR_{\rm{all}}}$, the $(\sH,
\sR)$-consistency bound implies Bayes-consistency.

We will prove $(\sH,
\sR)$-consistency bounds for our proposed surrogate losses, which
imply their Bayes-consistency. One key term in our bound is the \emph{minimizability gap},
defined as $\sM_{L}(\sH, \sR) = \sE^*_{L}(\sH, \sR) - \E_{x} \E_{y
  \mid x}\bracket*{L(h, r, x, y)}$. The minimizability gap
characterizes the difference between the best-in-class error and the
expected best-in-class point-wise error, and is non-negative. As shown
by \citet{mao2023cross}, the minimizability gap is upper bounded by
the approximation error, satisfying $0 \leq \sM_{L}(\sH, \sR) \leq
\sE^*_{L}(\sH, \sR) - \sE^*_{L}(\sH_{\rm{all}}, \sR_{\rm{all}})$ and
is generally a finer quantity. The minimizability gap
vanishes when $(\sH, \sR) =
(\sH_{\rm{all}}, \sR_{\rm{all}})$, or, more generally, when
$\sE^*_{L}(\sH, \sR) = \sE^*_{L}(\sH_{\rm{all}}, \sR_{\rm{all}})$.

Given a loss function $\ell \colon (r, x, y) \mapsto
\Rset_{+}$ that only depends on the hypothesis $r$, 
the notions of generalization error, best-in-class generalization
error, and minimizability gaps, as well as Bayes-consistency and
$\sR$-consistency bounds, are similarly defined.

In the next sections, we study the problem of learning a pair $(h, r)$
in the framework of regression with deferral. We will derive a family  of surrogate losses of $\ldef$, starting from first principles. We will
show that these loss functions benefit from strong consistency guarantees, which yield directly principled algorithms for our deferral
problem. We will specifically distinguish two approaches: the
single-stage surrogate losses, where the predictor $h$ and the
deferral function $r$ are jointly learned, and the two-stage surrogate
losses wherein the predictor $h$ have been previously trained and is
fixed and subsequently
used in the learning process of the deferral function $r$.

\section{Single-stage scenario}
\label{sec:single-stage}

In this section, we derive single-stage surrogate losses for the
deferral loss and prove their strong $(\sH, \sR)$-consistency bounds
guarantees. To do so, we first prove that the following alternative
expression holds for $\ldef$.

\begin{restatable}{lemma}{Ldef}
\label{lemma:Ldef}
For any $(h, r) \in \sH \times \sR$ and $(x, y) \in \sX \times \sY$,
the loss function $\ldef$ can be expressed as follows:
\ifdim\columnwidth=\textwidth
{
\begin{align*}
  \ldef(h, r, x, y) 
  & = \bracket*{\sum_{j = 1}^{\num} c_j(x,y)} 1_{\rr(x) \neq 0} 
  + \sum_{j = 1}^{\num} \bracket*{\sfL(h(x), y)
    + \sum_{k = 1}^{\num} c_k(x, y) 1_{k \neq j}} 1_{\rr(x) \neq j}\\
  & \quad - \paren*{\num - 1} \bracket*{\sfL(h(x), y) + \sum_{j = 1}^{\num} c_j(x, y)}.
\end{align*}
}
\else
{
\begin{align*}
  \ldef(h, r, x, y) 
  & = \bracket*{\sum_{j = 1}^{\num} c_j(x,y)} 1_{\rr(x) \neq 0} \\
  & + \sum_{j = 1}^{\num} \bracket*{\sfL(h(x), y)
    + \sum_{k = 1}^{\num} c_k(x, y) 1_{k \neq j}} 1_{\rr(x) \neq j}\\
  & - \paren*{\num - 1} \bracket*{\sfL(h(x), y) + \sum_{j = 1}^{\num} c_j(x, y)}.
\end{align*}
}
\fi
\end{restatable}
Let $\ell_{0-1}$ be the zero-one multi-class classification loss defined
by $\ell_{0-1}(r, x, y) = 1_{\rr(x) \neq y}$ for all $r \in \Rset$ and
$(x, y) \in \sX \times \sY$ and let $\ell \colon \sR \times \sX \times
[\num] \to \Rset_{+}$ be a surrogate loss for $\ell_{0-1}$ such that
$\ell \geq \ell_{0-1}$. $\ell$ may be chosen to be the logistic loss,
for example. Since the last term $\paren*{\num - 1} \sum_{j =
  1}^{\num} c_j(x,y)$ in the expression of $\ldef$ in
Lemma~\ref{lemma:Ldef} does not depend on $h$ and $r$, the following
loss function $L_{\ell}$ defined for all $(h, r) \in \sH \times \sR$
and $(x, y) \in \sX \times \sY$ by
\begin{align}
\label{eq:sur}
L_{\ell}(h, r, x, y) 
& = \bracket*{\sum_{j = 1}^{\num} c_j(x,y)} \ell(r, x, 0)+ \sum_{j = 1}^{\num} \bracket*{\sfL(h(x), y)
  + \sum_{j' \neq j}^{\num} c_{j'}(x,y)} \ell(r, x, j) \\
& \quad - \paren*{\num - 1} \sfL(h(x), y), \nonumber
\end{align}
is a natural single-stage surrogate loss for $\ldef$.  We will show
that when $\ell$ admits a strong $\sR$-consistency bound with respect
to $\ell_{0-1}$, then $L_{\ell}$ admits an $(\sH, \sR)$-consistency
bound with respect to $\ldef$.

Let us underscore the novelty of the surrogate loss formulation
presented in equation \eqref{eq:sur} in the context of learning to
defer with multiple experts. This formulation represents a substantial
departure from the existing score-based approach prevalent in
classification. As previously highlighted, the score-based formulation
becomes inapplicable in regression.  Our new predictor-rejector
formulation not only overcomes this limitation, but also provides the
foundation for the design of new deferral algorithms for
classification.

We say that a hypothesis set $\sR$ is \emph{regular} if for any $x \in
\sX$, the predictions made by the hypotheses in $\sR$ cover the
complete set of possible classification labels: $\curl*{\rr(x) \colon
  r\in \sR} = \curl*{0, 1, \ldots, \num}$. Widely used hypothesis sets
such as linear hypotheses, neural networks, and of course the family
of all measurable functions are all regular.

Recent studies by \citet{awasthi2022multi} and \citet{mao2023cross}
demonstrate that common multi-class surrogate losses, such as
constrained losses and comp-sum losses (including the logistic loss),
admit strong $\sR$-consistency bounds with respect to the multi-class
zero-one loss $\ell_{0-1}$, when using such regular hypothesis sets.
The next result shows that, for multi-class loss functions $\ell$,
their corresponding deferral surrogate losses $L_{\ell}$
(Eq.~\eqref{eq:sur}) also exhibit $(\sH, \sR)$-consistency bounds with
respect to the deferral loss (Eq.~\eqref{eq:def}).

\begin{restatable}{theorem}{Single}
\label{thm:single}
Let $\sR$ be a regular hypothesis set and $\ell$ a surrogate loss for
the multi-class loss function $\ell_{0-1}$ upper-bounding
$\ell_{0-1}$. Assume that there exists a function $\Gamma(t) = \beta\,
t^{\alpha}$ for some $\alpha \in (0, 1]$ and $\beta > 0$, such that
  the following $\sR$-consistency bound holds for all $r \in \sR$ and
  any distribution,
\ifdim\columnwidth=\textwidth
{
\begin{equation*}
  \sE_{\ell_{0-1}}(r) - \sE^*_{\ell_{0-1}}(\sR) + \sM_{\ell_{0-1}}(\sR)\\
  \leq \Gamma\paren*{\sE_{\ell}(r)
    - \sE^*_{\ell}(\sR) + \sM_{\ell}(\sR)}.
\end{equation*}
}
\else{
\begin{multline*}
  \sE_{\ell_{0-1}}(r) - \sE^*_{\ell_{0-1}}(\sR) + \sM_{\ell_{0-1}}(\sR)\\
  \leq \Gamma\paren*{\sE_{\ell}(r)
    - \sE^*_{\ell}(\sR) + \sM_{\ell}(\sR)}.
\end{multline*}
}
\fi
Then, the following $(\sH,\sR)$-consistency bound holds for all $h\in
\sH$, $r\in \sR$ and any distribution,
\ifdim\columnwidth=\textwidth
{
\begin{equation*}
   \sE_{\ldef}(h, r) - \sE_{\ldef}^*(\sH,\sR) + \sM_{\ldef}(\sH,\sR)\\
   \leq 
  \ov \Gamma\paren*{\sE_{L_{\ell}}(h, r)
    -  \sE_{L_{\ell}}^*(\sH,\sR) + \sM_{L_{\ell}}(\sH,\sR)},
\end{equation*}
}
\else{
\begin{multline*}
   \sE_{\ldef}(h, r) - \sE_{\ldef}^*(\sH,\sR) + \sM_{\ldef}(\sH,\sR)\\
   \leq 
  \ov \Gamma\paren*{\sE_{L_{\ell}}(h, r)
    -  \sE_{L_{\ell}}^*(\sH,\sR) + \sM_{L_{\ell}}(\sH,\sR)},
\end{multline*}
}
\fi
where $\ov \Gamma(t)
= \max\curl*{t, \paren*{\num\paren*{\ul
      + \sum_{j = 1}^{\num}\uc_j}}^{1 - \alpha} \beta\, t^{\alpha}}$.
\end{restatable}
The proof is given in Appendix~\ref{app:exp}. As already mentioned,
when the best-in-class error coincides with the Bayes error
$\sE^*_{L}(\sH, \sR) = \sE^*_{L}\paren*{\sH_{\rm{all}},
  \sR_{\rm{all}}}$ for $L = L_{\ell}$ and $L = \ldef$, the
minimizability gaps $\sM_{L_{\ell}}(\sH,\sR)$ and
$\sM_{\ldef}(\sH,\sR)$ vanish. In such cases, the
$(\sH, \sR)$-consistency bound guarantees that when the surrogate
estimation error $\sE_{L_{\ell}}(h, r) - \sE_{L_{\ell}}^*(\sH,\sR)$ is
optimized up to $\e$, the estimation error of the deferral loss $
\sE_{\ldef}(h, r) - \sE_{\ldef}^*(\sH,\sR)$ is upper bounded by $\ov
\Gamma(\e)$.

In particular, when both $\sH$ and $\sR$ include all measurable
functions, all the minimizability gap terms in
Theorem~\ref{thm:single} vanish, which yields the following result.
\begin{corollary}
\label{cor:single}
Given a multi-class loss function $\ell \geq \ell_{0-1}$. Assume that
there exists a function $\Gamma(t) = \beta\, t^{\alpha}$ for some
$\alpha \in (0, 1]$ and $\beta > 0$, such that the following excess
  error bound holds for all $r \in \sR_{\rm{all}}$ and any
  distribution,
\begin{equation*}
  \sE_{\ell_{0-1}}(r) - \sE^*_{\ell_{0-1}}(\sR_{\rm{all}})
  \leq \Gamma\paren*{\sE_{\ell}(r) - \sE^*_{\ell}(\sR_{\rm{all}})}.
\end{equation*}
Then, the following excess error bound holds for all $h\in \sH_{\rm{all}}$, $r\in \sR_{\rm{all}}$ and any distribution,
\begin{equation*}
   \sE_{\ldef}(h, r) - \sE_{\ldef}^*(\sH_{\rm{all}},\sR_{\rm{all}})\\
   \leq 
    \ov \Gamma\paren*{\sE_{L_{\ell}}(h, r) -  \sE_{L_{\ell}}^*(\sH_{\rm{all}},\sR_{\rm{all}})},
\end{equation*}
where $\ov \Gamma(t) = \max\curl*{t, \paren*{\num\paren*{\ul + \sum_{j = 1}^{\num}\uc_j}}^{1 - \alpha} \beta\, t^{\alpha}}$.
\end{corollary}
In this case, as shown by \citet{mao2023cross}, $\Gamma(t)$ can be
expressed as $\sqrt{2t}$ for the logistic loss $\ell_{\rm{\log}}
\colon (r, x, y) \mapsto \log_2 \paren*{\sum_{j = 0}^{\num} e^{r(x, j)
    - r(x, y)}}$. Then, by Corollary~\ref{cor:single}, we further
obtain the following corollary.
\begin{corollary}
\label{cor:single-example}
For any $h\in \sH_{\rm{all}}$, $r\in \sR_{\rm{all}}$ and any distribution,
\begin{equation*}
    \sE_{\ldef}(h, r) - \sE_{\ldef}^*(\sH_{\rm{all}},\sR_{\rm{all}})\\
    \leq 
    \ov\Gamma\paren*{\sE_{L_{\ell_{\rm{log}}}}(h, r) -  \sE_{L_{\ell_{\rm{log}}}}^*(\sH_{\rm{all}},\sR_{\rm{all}})},
\end{equation*}
where $\ov \Gamma (t) = \max\curl*{t, \sqrt{2 \num} \paren*{\ul +
    \sum_{j = 1}^{\num}\uc_j}^{\frac12} t^{\frac 12}}$.
\end{corollary}
By taking the limit on both sides, we derive the Bayes-consistency of
these single-stage surrogate losses $L_{\ell}$ with respect to the
deferral loss $\ldef$. More generally, Corollary~\ref{cor:single}
shows that $L_{\ell}$ admits an excess error bound with respect to
$\ldef$ when $\ell$ admits an excess error bound with respect to
$\ell_{0-1}$.

\section{Two-stage scenario}
\label{sec:two-stage}

In the single-stage scenario, we introduced a family of surrogate losses and resulting algorithms for effectively learning the pair $(h, r)$. However, practical applications often encounter a \emph{two-stage scenario}, where deferral decisions are based on a fixed, pre-trained predictor $h$. Retraining this predictor is often prohibitively expensive or time-consuming. Thus, this two-stage scenario \citep{MaoMohriZhong2023} requires a different approach to optimize deferral decisions controlled by $r$, while using the existing predictor
$h$.

In this section, we will introduce a principled
two-stage algorithm for regression with deferral, with favorable
consistency guarantees. Remarkably, we show that the single-stage approach can be adapted for the two-stage scenario if we fix the predictor $h$ and disregard constant
terms.

Let $h$ be a predictor learned by minimizing a regression loss
$\sfL$ in a first stage. A deferral function $r$ is then learned based on
that predictor $h$ and the following loss function $L_{\ell}^h$ in the
second stage: for any $r \in \sR$, $x \in \sX$ and $y \in \sY$,
\begin{align}
\label{eq:two-stage-sur}
  L_{\ell}^h (r, x, y)
  = \bracket*{\sum_{j = 1}^{\num} c_j(x,y)} \, \ell(r, x, 0) + \sum_{j = 1}^{\num} \bracket*{\sfL(h(x), y)
    + \sum_{j' \neq j}^{\num} c_{j'}(x,y)} \, \ell(r, x, j),
\end{align}
where $\ell$ is a surrogate loss in the standard multi-class
classification.  
Equation \eqref{eq:two-stage-sur} resembles \eqref{eq:sur}, except for the constant term $(\num - 1)\sfL(h(x), y)$. In \eqref{eq:two-stage-sur}, the predictor $h$ remains fixed while only the deferral function $r$ is optimized. In \eqref{eq:sur}, both $h$ and $r$ are learned jointly.
\ignore{
Note that \eqref{eq:two-stage-sur} has a same
formulation as \eqref{eq:sur} modulo the constant term $(\num -
1)\sfL(h(x), y)$ in the second-stage. However, the main difference is
that in \eqref{eq:two-stage-sur}, $h$ is fixed and only $r$ is
learned, while both $h$ and $r$ are jointly learned in
\eqref{eq:sur}. 
}

Similarly, we define $\ldef^h$ as the
deferral loss \eqref{eq:def} with a fixed predictor $h$ as follows:
\begin{equation}
\label{eq:two-stage-deferral-loss}
\mspace{-1mu}
\ldef^h (r, x, y) = \sfL(h(x), y) 1_{\rr(x) = 0} + \sum_{j = 1}^{\num} c_j(x,y) 1_{\rr(x) = j}.
\mspace{-3mu}
\end{equation}
Here too, $h$ is fixed in \eqref{eq:two-stage-deferral-loss}. Both $L_{\ell}^h$ and $\ldef^h$ are loss functions defined for deferral function $r$, while $\ell_{\ell}$ and $\ell_{\rm{def}}$ are loss functions defined for pairs $(h, r) \in (\sH, \sR)$.

As with the proposed single-stage approach, the two-stage surrogate losses $L_{\ell}^h$ benefit from strong consistency guarantees. We show that in the second stage where a predictor $h$ is fixed, the surrogate loss function $L_{\ell}^h$ benefits from $\sR$-consistency bounds with respect to $\ldef^h$ when $\ell$ admits a strong $\sR$-consistency bound with respect to the binary zero-one loss $\ell_{0-1}$. 
\begin{restatable}{theorem}{TwostageR}
\label{thm:tsr}
Given a hypothesis set $\sR$, a multi-class loss function $\ell \geq \ell_{0-1}$ and a predictor $h$. Assume that there exists a function $\Gamma(t) = \beta\, t^{\alpha}$ for some $\alpha \in (0, 1]$ and $\beta > 0$, such that the following $\sR$-consistency bound holds for all $r \in \sR$ and any distribution,
\ifdim\columnwidth=\textwidth
{
\begin{align*}
\sE_{\ell_{0-1}}(r) - \sE^*_{\ell_{0-1}}(\sR) + \sM_{\ell_{0-1}}(\sR) \leq \Gamma\paren*{\sE_{\ell}(r) - \sE^*_{\ell}(\sR) + \sM_{\ell}(\sR)}.
\end{align*}
}
\else
{
\begin{align*}
& \sE_{\ell_{0-1}}(r) - \sE^*_{\ell_{0-1}}(\sR) + \sM_{\ell_{0-1}}(\sR)\\
& \quad \leq \Gamma\paren*{\sE_{\ell}(r) - \sE^*_{\ell}(\sR) + \sM_{\ell}(\sR)}.
\end{align*}
}
\fi
Then, the following $\sR$-consistency bound holds for all $r\in \sR$ and any distribution,
\ifdim\columnwidth=\textwidth
{
\begin{align*}
   \sE_{L^h_{\rm{def}}}(r) - \sE_{L^h_{\rm{def}}}^*(\sR) + \sM_{L^h_{\rm{def}}}(\sR) \leq 
    \ov \Gamma\paren*{\sE_{L^h_{\ell}}(r) -  \sE_{L^h_{\ell}}^*(\sR) + \sM_{L^h_{\ell}}(\sR)},
\end{align*}
}
\else
{
\begin{align*}
   & \sE_{L^h_{\rm{def}}}(r) - \sE_{L^h_{\rm{def}}}^*(\sR) + \sM_{L^h_{\rm{def}}}(\sR)\\
   & \quad \leq 
    \ov \Gamma\paren*{\sE_{L^h_{\ell}}(r) -  \sE_{L^h_{\ell}}^*(\sR) + \sM_{L^h_{\ell}}(\sR)},
\end{align*}
}
\fi
where $\ov \Gamma(t) = \paren*{\num\paren*{\ul + \sum_{j = 1}^{\num}\uc_j}}^{1 - \alpha} \beta\, t^{\alpha}$.
\end{restatable}
The proof is given in Appendix~\ref{app:tsr}. When the best-in-class error coincides with the Bayes error, $\sE^*_{L}(\sR) = \sE^*_{L}\paren*{\sR_{\rm{all}}}$ for $L = L^h_{\ell}$ and $L = L^h_{\rm{def}}$, the minimizability gaps $\sM_{L^h_{\ell}}(\sR)$ and $\sM_{L^h_{\rm{def}}}(\sR)$ vanish. In that case, the  $\sR$-consistency bound guarantees that when the surrogate estimation error $\sE_{L^h_{\ell}}(r) -  \sE_{L^h_{\ell}}^*(\sR)$ is optimized up to $\e$, the target estimation error $ \sE_{L^h_{\rm{def}}}(r) - \sE_{L^h_{\rm{def}}}^*(\sR)$ is upper bounded by $\ov \Gamma(\e)$. In the special case where $\sH$ and $\sR$ are the family of all measurable functions, all the minimizability gap terms in Theorem~\ref{thm:tsr} vanish. Thus, we obtain the following corollary.
\begin{corollary}
\label{cor:tsr}
Given a multi-class loss function $\ell \geq \ell_{0-1}$ and a predictor $h$. Assume that there exists a function $\Gamma(t) = \beta\, t^{\alpha}$ for some $\alpha \in (0, 1]$ and $\beta > 0$, such that the following excess error bound holds for all $r \in \sR$ and any distribution,
\begin{equation*}
\sE_{\ell_{0-1}}(r) - \sE^*_{\ell_{0-1}}(\sR_{\rm{all}}) \leq \Gamma\paren*{\sE_{\ell}(r) - \sE^*_{\ell}(\sR_{\rm{all}})}.
\end{equation*}
Then, the following excess error bound holds for all $r\in \sR_{\rm{all}}$ and any distribution,
\begin{equation}
\label{eq:bound-tsr-all}
   \sE_{L^h_{\rm{def}}}(r) - \sE_{L^h_{\rm{def}}}^*(\sR_{\rm{all}})) \leq 
    \ov \Gamma\paren*{\sE_{L^h_{\ell}}(r) -  \sE_{L^h_{\ell}}^*(\sR_{\rm{all}})},
\end{equation}
where $\ov \Gamma(t) = \paren*{\num\paren*{\ul + \sum_{j = 1}^{\num}\uc_j}}^{1 - \alpha} \beta\, t^{\alpha}$.
\end{corollary}
Corollary~\ref{cor:tsr} shows that $L^h_{\ell}$ admits an excess error bound with respect to $L^h_{\rm{def}}$ when $\ell$ admits an excess error bound with respect to $\ell_{0-1}$.

We now establish $(\sH, \sR)$-consistency bounds the entire two-stage approach with respect to the deferral loss function $\ldef$. This result applies to any multi-class loss function $\ell$ that satisfies a strong $\sR$-consistency bound with respect to the multi-class zero-one loss $\ell_{0-1}$.
\begin{restatable}{theorem}{TwostageHR}
\label{thm:tshr}
Given a hypothesis set $\sH$, a regular hypothesis set $\sR$ and a multi-class loss function $\ell \geq \ell_{0-1}$. Assume that there exists a function $\Gamma(t) = \beta\, t^{\alpha}$ for some $\alpha \in (0, 1]$ and $\beta > 0$, such that the following $\sR$-consistency bound holds for all $r \in \sR$ and any distribution,
\ifdim\columnwidth=\textwidth
{
\begin{align*}
\sE_{\ell_{0-1}}(r) - \sE^*_{\ell_{0-1}}(\sR) + \sM_{\ell_{0-1}}(\sR) \leq \Gamma\paren*{\sE_{\ell}(r) - \sE^*_{\ell}(\sR) + \sM_{\ell}(\sR)}.
\end{align*}
}
\else
{
\begin{align*}
& \sE_{\ell_{0-1}}(r) - \sE^*_{\ell_{0-1}}(\sR) + \sM_{\ell_{0-1}}(\sR)\\
& \quad \leq \Gamma\paren*{\sE_{\ell}(r) - \sE^*_{\ell}(\sR) + \sM_{\ell}(\sR)}.
\end{align*}
}
\fi
Then, the following $(\sH,\sR)$-consistency bound holds for all $h \in \sH$, $r \in \sR$ and any distribution,
\ifdim\columnwidth=\textwidth
{
\begin{align*}
\sE_{\ldef}(h, r) - \sE_{\ldef}^*(\sH,\sR) + \sM_{\ldef}(\sH,\sR) 
& \leq \sE_{\sfL}(h) - \sE_{\sfL}(\sH) + \sM_{\sfL}(\sH)\\
& \quad + \ov \Gamma\paren*{\sE_{L^h_{\ell}}(r) -  \sE_{L^h_{\ell}}^*(\sR) + \sM_{L^h_{\ell}}(\sR)},
\end{align*}
}
\else
{
\begin{equation}
\begin{aligned}
\label{eq:bound-tshr}
   & \sE_{\ldef}(h, r) - \sE_{\ldef}^*(\sH,\sR) + \sM_{\ldef}(\sH,\sR)\\
   & \quad \leq \sE_{\sfL}(h) - \sE_{\sfL}(\sH) + \sM_{\sfL}(\sH)\\
   & \qquad + \ov \Gamma\paren*{\sE_{L^h_{\ell}}(r) -  \sE_{L^h_{\ell}}^*(\sR) + \sM_{L^h_{\ell}}(\sR)},
\end{aligned}
\end{equation}
}
\fi
where $\ov \Gamma(t) = \paren*{\num\paren*{\ul + \sum_{j = 1}^{\num}\uc_j}}^{1 - \alpha}\beta\, t^{\alpha}$.
\end{restatable}
The proof is presented in Appendix~\ref{app:tshr}. As before, when $\sH$ and $\sR$ are the family of all measurable functions, all the minimizability gap terms in Theorem~\ref{thm:tshr} vanish. In particular, $\Gamma(t)$ can be expressed as $\sqrt{2t}$ for the logistic loss. Thus, we obtain the following on excess error bounds.
\begin{corollary}
\label{cor:tshr}
Given a multi-class loss function $\ell \geq \ell_{0-1}$. Assume that there exists a function $\Gamma(t) = \beta\, t^{\alpha}$ for some $\alpha \in (0, 1]$ and $\beta > 0$, such that the following excess error bound holds for all $r \in \sR_{\rm{all}}$ and any distribution,
\begin{equation*}
\sE_{\ell_{0-1}}(r) - \sE^*_{\ell_{0-1}}(\sR_{\rm{all}}) \leq \Gamma\paren*{\sE_{\ell}(r) - \sE^*_{\ell}(\sR_{\rm{all}})}.
\end{equation*}
Then, the following excess error bound holds for all $h\in \sH_{\rm{all}}$, $r\in \sR_{\rm{all}}$ and any distribution,
\begin{equation}
\begin{aligned}
\label{eq:bound-tshr-all}
   \sE_{\ldef}(h, r) - \sE_{\ldef}^*(\sH_{\rm{all}},\sR_{\rm{all}})\leq \sE_{\sfL}(h) - \sE_{\sfL}(\sH_{\rm{all}}) + \ov \Gamma\paren*{\sE_{L^h_{\ell}}(r) -  \sE_{L^h_{\ell}}^*(\sR_{\rm{all}})},
\end{aligned}
\end{equation}
where $\ov \Gamma(t) = \paren*{\num\paren*{\ul + \sum_{j = 1}^{\num}\uc_j}}^{1 - \alpha} \beta\, t^{\alpha}$. In particular, $\ov \Gamma (t) = \sqrt{2 \num} \paren*{\ul + \sum_{j = 1}^{\num}\uc_j}^{\frac12} t^{\frac 12}$ for $\ell = \ell_{\rm{log}}$.
\end{corollary}
Corollary~\ref{cor:tshr} shows that our two-stage approach admits an excess error bound with respect to $\ldef$ when $\ell$ admits an excess error bound with respect to $\ell_{0-1}$. More generally, when the minimizability gaps are zero, as when the best-in-class errors coincide with the Bayes errors, the $(\sH, \sR)$-consistency bound of Theorem~\ref{thm:tshr} guarantees that the target estimation error, $\sE_{\ldef}(h, r) - \sE_{\ldef}^*(\sH,\sR)$, is upper bounded by $\e_1 + \ov \Gamma(\e_2)$
provided that the surrogate estimation error in the first stage, $\sE_{\sfL}(h) - \sE_{\sfL}(\sH)$, is reduced to $\e_1$ and the surrogate estimation error in the second stage, $\sE_{L^h_{\ell}}(r) -  \sE_{L^h_{\ell}}^*(\sR)$, reduced to $\e_2$.

\section{Special case of a single expert}
\label{sec:single-expert}

In the special case of a single expert, $\num = 1$, both the single-stage surrogate loss $L_{\ell}$ and the two-stage surrogate loss $L^h_{\ell}$ can be simplified as follows: 
\begin{equation*}
c(x,y) \ell(r, x, 0) + \sfL(h(x), y) \ell(r, x, 1).
\end{equation*}
Let $\ell(r, x, 0) = \Phi(r(x))$ and $\ell(r, x, 1) = \Phi(-r(x))$, where $\Phi\colon \Rset \to \Rset_{+}$ is a non-increasing auxiliary function upper bounding the indicator $u \mapsto 1_{u \leq 0}$. Here, $r \colon \sX \to \Rset$ is a function whose sign determines if there is
deferral, that is $r(x) \leq 0$:
\begin{equation*}
 \ell_{\rm{def}}(h, r, x, y) = \sfL(h(x), y) 1_{r(x) > 0} + c(x,y) 1_{r(x) \leq 0}.
\end{equation*}
As an example, $\Phi$ can be the auxiliary function that defines a margin-based loss in the binary classification. Thus, both the single-stage surrogate loss $\ell_{\Phi}$ and the two-stage surrogate loss $\ell^h_{\Phi}$ can be reformulated as follows: 
\begin{equation}
\label{eq:single}
c(x,y) \Phi(r(x)) + \sfL(h(x), y) \Phi(-r(x)).
\end{equation}
Some common examples of $\Phi$ are listed in Table~\ref{tab:sur-binary} in Appendix~\ref{app:sur-binary}.
Note that \eqref{eq:single} is a straightforward extension of the two-stage formulation given in \citep[Eq.~(5)]{MaoMohriZhong2024predictor}. In their formulation, the zero-one loss function replaces the regression loss and is tailored for the classification context. A special case of the straightforward extension \eqref{eq:single} is one where the cost does not depend on the label $y$ and the squared loss is considered. This coincides with the loss function \citep[Eq.~(10)]{cheng2023regression} in the context of regression with abstention. It is important to note that incorporating $y$ as argument of the cost functions is crucial in the more general deferral setting, as each cost takes into account the accuracy of the corresponding expert. 

Let the binary zero-one loss be $\ell^{\rm{bi}}_{0-1}(r, x, y) = 1_{\sign\paren*{r(x)} \neq y}$, where $\sign(\alpha) = 1_{\alpha > 0} - 1_{\alpha \leq 0}$. We say that a hypothesis set $\sR$ consists of functions mapping from $\sX$ to $\Rset$ is \emph{regular}, if  $\curl*{\sign \paren*{r(x)} \colon r \in \sR} = \curl*{\plus 1, \minus 1}$ for any $x \in \sX$.

Then, Theorems~\ref{thm:single} and ~\ref{thm:tshr} can be reduced to Theorems~\ref{thm:single-binary} and ~\ref{thm:tshr-binary} below respectively. We present these guarantees and their corresponding corollaries in the following sections.

\subsection{Single-stage guarantees}

Here, we present guarantees for the single-stage surrogate.
\begin{restatable}{theorem}{SingleBinary}
\label{thm:single-binary}
Given a hypothesis set $\sH$, a regular hypothesis set $\sR$ and a margin-based loss function $\Phi$. Assume that there exists a function $\Gamma(t) = \beta\, t^{\alpha}$ for some $\alpha \in (0, 1]$ and $\beta > 0$, such that the following $\sR$-consistency bound holds for all $r \in \sR$ and any distribution,
\begin{align*}
\sE_{\ell^{\rm{bi}}_{0-1}}(r) - \sE^*_{\ell^{\rm{bi}}_{0-1}}(\sR) + \sM_{\ell^{\rm{bi}}_{0-1}}(\sR) \leq \Gamma\paren*{\sE_{\Phi}(r) - \sE^*_{\Phi}(\sR) + \sM_{\Phi}(\sR)}.
\end{align*}
Then, the following $(\sH,\sR)$-consistency bound holds for all $h\in \sH$, $r\in \sR$ and any distribution,
\begin{align*}
   \sE_{\ell_{\rm{def}}}(h, r) - \sE_{\ell_{\rm{def}}}^*(\sH,\sR) + \sM_{\ell_{\rm{def}}}(\sH,\sR) \leq 
    \ov \Gamma\paren*{\sE_{\ell_{\Phi}}(h, r) -  \sE_{\ell_{\Phi}}^*(\sH,\sR) + \sM_{\ell_{\Phi}}(\sH,\sR)},
\end{align*}
where $\ov \Gamma(t) = \max\curl*{t, \paren*{\ul + \uc}^{1 - \alpha} \beta\, t^{\alpha}}$.
\end{restatable}
In particular, when $\sH$ and $\sR$ are the family of all measurable functions, all the minimizability gap terms in Theorem~\ref{thm:single-binary} vanish. In this case, as shown by \citet{awasthi2022h}, $\Gamma(t)$ can be expressed as $\frac{t^2}{2}$ for exponential and logistic loss, $t^2$ for quadratic loss and $t$ for hinge, sigmoid and $\rho$-margin losses. Thus, the following result holds.
\begin{corollary}
\label{cor:single-binary}
Given a margin-based loss function $\Phi$. Assume that there exists a function $\Gamma(t) = \beta\, t^{\alpha}$ for some $\alpha \in (0, 1]$ and $\beta > 0$, such that the following excess error bound holds for all $r \in \sR_{\rm{all}}$ and any distribution,
\begin{equation*}
\sE_{\ell^{\rm{bi}}_{0-1}}(r) - \sE^*_{\ell^{\rm{bi}}_{0-1}}(\sR_{\rm{all}}) \leq \Gamma\paren*{\sE_{\Phi}(r) - \sE^*_{\Phi}(\sR_{\rm{all}})}.
\end{equation*}
Then, the following excess error bound holds for all $h\in \sH_{\rm{all}}$, $r\in \sR_{\rm{all}}$ and any distribution,
\begin{align*}
   \sE_{\ell_{\rm{def}}}(h, r) - \sE_{\ell_{\rm{def}}}^*(\sH_{\rm{all}},\sR_{\rm{all}})\leq 
    \ov \Gamma\paren*{\sE_{\ell_{\Phi}}(h, r) -  \sE_{\ell_{\Phi}}^*(\sH_{\rm{all}},\sR_{\rm{all}})},
\end{align*}
where $\ov \Gamma(t) = \max\curl*{t, \paren*{\ul + \uc}^{1 - \alpha} \beta\, t^{\alpha}}$. In particular,  $\ov \Gamma (t) = \max\curl*{t, \frac12 \paren{\ul + \uc}^{\frac12} t^{\frac 12}}$ for $\Phi = \Phi_{\rm{exp}}$ and $\Phi_{\rm{log}}$, $\ov \Gamma (t) = \max\curl*{t, \paren{\ul + \uc}^{\frac12} t^{\frac 12}}$ for $\Phi = \Phi_{\rm{quad}}$, and $\ov \Gamma (t) = t$ for $\Phi = \Phi_{\rm{hinge}}$, $\Phi_{\rm{sig}}$, and $\Phi_{\rho}$.
\end{corollary}
By taking the limit on both sides, we derive the Bayes-consistency and excess error bound of these single-stage surrogate losses $\ell_{\Phi}$ with respect to the deferral loss $\ell_{\rm{def}}$. More generally, Corollary~\ref{cor:single-binary} shows that $\ell_{\Phi}$ admits an excess error bound with respect to $\ell_{\rm{def}}$ when $\Phi$ admits an excess error bound with respect to $\ell_{0-1}$. Corollary~\ref{cor:single-binary} also include the theoretical guarantees in \citep[Theorems~7 and 8]{cheng2023regression} as a special case where the cost does not depend on the label $y$ and the squared loss is considered.

\subsection{Two-stage guarantee}
Here, we present guarantees for the two-stage surrogate.
\begin{restatable}{theorem}{TwostageHRBinary}
\label{thm:tshr-binary}
Given a hypothesis set $\sH$, a regular hypothesis set $\sR$ and a margin-based loss function $\Phi$. Assume that there exists a function $\Gamma(t) = \beta\, t^{\alpha}$ for some $\alpha \in (0, 1]$ and $\beta > 0$, such that the following $\sR$-consistency bound holds for all $r \in \sR$ and any distribution,
\begin{align*}
\sE_{\ell^{\rm{bi}}_{0-1}}(r) - \sE^*_{\ell^{\rm{bi}}_{0-1}}(\sR) + \sM_{\ell^{\rm{bi}}_{0-1}}(\sR) \leq \Gamma\paren*{\sE_{\Phi}(r) - \sE^*_{\Phi}(\sR) + \sM_{\Phi}(\sR)}.
\end{align*}
Then, the following $(\sH,\sR)$-consistency bound holds for all $h\in \sH$, $r\in \sR$ and any distribution,
\begin{equation*}
\begin{aligned}
   \sE_{\ell_{\rm{def}}}(h, r) - \sE_{\ell_{\rm{def}}}^*(\sH,\sR) + \sM_{\ell_{\rm{def}}}(\sH,\sR)
   & \leq \sE_{\sfL}(h) - \sE_{\sfL}(\sH) + \sM_{\sfL}(\sH)\\
   & \quad + \ov \Gamma\paren*{\sE_{\ell^h_{\Phi}}(r) -  \sE_{\ell^h_{\Phi}}^*(\sR) + \sM_{\ell^h_{\Phi}}(\sR)},
\end{aligned}
\end{equation*}
where $\ov \Gamma(t) = \paren*{\ul + \uc}^{1 - \alpha} \beta\, t^{\alpha}$.
\end{restatable}
 As before, when $\sH$ and $\sR$ include all measurable functions, all the minimizability gap terms in Theorem~\ref{thm:tshr} vanish. In particular, $\Gamma(t)$ can be expressed as $\frac{t^2}{2}$ for exponential and logistic loss, $t^2$ for quadratic loss and $t$ for hinge, sigmoid and $\rho$-margin losses \citep{awasthi2022h}. Thus, we obtain the following result.
\begin{corollary}
\label{cor:tshr-binary}
Given a margin-based loss function $\Phi$. Assume that there exists a function $\Gamma(t) = \beta\, t^{\alpha}$ for some $\alpha \in (0, 1]$ and $\beta > 0$, such that the following excess error bound holds for all $r \in \sR_{\rm{all}}$ and any distribution,
\begin{equation*}
\sE_{\ell^{\rm{bi}}_{0-1}}(r) - \sE^*_{\ell^{\rm{bi}}_{0-1}}(\sR_{\rm{all}}) \leq \Gamma\paren*{\sE_{\Phi}(r) - \sE^*_{\Phi}(\sR_{\rm{all}})}.
\end{equation*}
Then, the following excess error bound holds for all $h\in \sH_{\rm{all}}$, $r\in \sR_{\rm{all}}$ and any distribution,
\begin{equation}
\begin{aligned}
   \sE_{\ell_{\rm{def}}}(h, r) - \sE_{\ell_{\rm{def}}}^*(\sH_{\rm{all}},\sR_{\rm{all}}) \leq \sE_{\sfL}(h) - \sE_{\sfL}(\sH_{\rm{all}}) + \ov \Gamma\paren*{\sE_{\ell^h_{\Phi}}(r) -  \sE_{\ell^h_{\Phi}}^*(\sR_{\rm{all}})},
\end{aligned}
\end{equation}
where $\ov \Gamma(t) = \paren*{\ul + \uc}^{1 - \alpha} \beta\, t^{\alpha}$. In particular, $\ov \Gamma (t) = \frac12 \paren{\ul + \uc}^{\frac12} t^{\frac 12}$ for $\Phi = \Phi_{\rm{exp}}$ and $\Phi_{\rm{log}}$, $\ov \Gamma (t) = \paren{\ul + \uc}^{\frac12} t^{\frac 12}$ for $\Phi = \Phi_{\rm{quad}}$, and $\ov \Gamma (t) = t$ for $\Phi = \Phi_{\rm{hinge}}$, $\Phi_{\rm{sig}}$, and $\Phi_{\rho}$.
\end{corollary}
Corollary~\ref{cor:tshr-binary} shows that the proposed two-stage approach admits an excess error bound with respect to $\ell_{\rm{def}}$ when $\Phi$ admits an excess error bound with respect to $\ell_{0-1}$. More generally, in the cases where the minimizability gaps are zero (as when the best-in-class errors coincide with the Bayes errors), the $(\sH, \sR)$-consistency bound in Theorem~\ref{thm:tshr-binary} guarantees that when the surrogate estimation error in the first stage $\sE_{\sfL}(h) - \sE_{\sfL}(\sH)$ is minimized up to $\e_1$ and the surrogate estimation error in the second stage $\sE_{\ell^h_{\Phi}}(r) -  \sE_{\ell^h_{\Phi}}^*(\sR)$ is minimized up to $\e_2$, the target estimation error $ \sE_{\ell_{\rm{def}}}(h, r) - \sE_{\ell_{\rm{def}}}^*(\sH,\sR)$ is upper bounded by $\e_1 + \ov \Gamma(\e_2)$.

\begin{table*}[t]
  \caption{System MSE of deferral with multiple experts:
    mean $\pm$ standard deviation over three runs.}
 \label{tab:deferral-1}
 \centering
 \begin{tabular}{@{\hspace{0cm}}lllllll@{\hspace{0cm}}}
  \toprule
 Dataset & Base cost & Method & Base model & Single expert & Two experts & Three experts \\
  \midrule
  \multirow{4}{*}{\texttt{Airfoil}} & \xmark & Single & ---              & $18.98 \pm 2.44$ & $13.16 \pm 0.93$ & \bm{$8.53 \pm 1.57$} \\
                           & \xmark & Two    & $23.35 \pm 1.90$ & $18.64 \pm 1.96$ & $13.33 \pm 0.92 $ & \bm{$8.81 \pm 1.56$} \\
                           & \cmark & Single & ---              & $18.83 \pm 2.14$ & $13.79 \pm 0.75 $ & \bm{$8.64 \pm 1.40$} \\
                           & \cmark & Two    & $23.35 \pm 1.90$ & $19.15 \pm 1.99$ & $15.12 \pm 0.62 $ & \bm{$10.06 \pm 1.54$}\\
  \midrule
  \multirow{4}{*}{\texttt{Housing}} & \xmark & Single & ---              & $14.85 \pm 5.40$ & $14.75 \pm 3.53$ & \bm{$12.43 \pm 2.03$} \\
                           & \xmark & Two    & $22.72 \pm 7.68$ & $16.26 \pm 5.58$ & $14.82 \pm 3.60$ & \bm{$12.02 \pm 1.97$} \\
                           & \cmark & Single & ---              & $15.17 \pm 5.18$ & $15.07 \pm 2.88$ & \bm{$14.80 \pm 3.48$} \\
                           & \cmark & Two    & $22.72 \pm 7.68$ & $16.24 \pm 4.64$ & $15.62 \pm 3.04$ & \bm{$14.87 \pm 4.04$} \\
  \midrule
  \multirow{4}{*}{\texttt{Concrete}}& \xmark & Single & ---              & $104.38 \pm 5.55$ & $41.08 \pm 2.05$ & \bm{$37.83 \pm 2.60$} \\
                           & \xmark & Two    & $120.20 \pm 8.09$& $114.73 \pm 6.50$ & $44.46 \pm 5.34$ & \bm{$36.75 \pm 1.76$} \\
                           & \cmark & Single & ---              & $105.01 \pm 5.40$ & $39.52 \pm 2.81$ & \bm{$38.46 \pm 1.79$} \\
                           & \cmark & Two    & $120.20 \pm 8.09$& $114.11 \pm 5.34$ & $39.93 \pm 2.77$ & \bm{$37.51 \pm 2.32$} \\
  \bottomrule
 \end{tabular}
\end{table*}

\section{Experiments}
\label{sec:experiments}

In this section, we report the empirical results for our single-stage and two-stage algorithms for regression with deferral on three datasets from the UCI machine learning repository \citep{asuncion2007uci}, the \texttt{Airfoil}, \texttt{Housing} and \texttt{Concrete}, which have also been studied in \citep{cheng2023regression}.  Since our work is the first to study regression with multi-expert deferral, there are no existing baselines to compare with.

\textbf{Setup and Metrics.} 
For each dataset, we randomly split it into a training set of 60\% examples, a validation set of 20\% examples and a test set of 20\% examples. We report results averaged over three such random splits. We adopted linear models for both the predictor $h$ and the deferral function $r$. We considered three experts
$g_1$, $g_2$ and $g_3$, each trained by feedforward neural networks
with ReLU activation functions \citep{nair2010rectified}
with one, two, and three hidden layers, respectively.
We used the Adam optimizer
\citep{kingma2014adam} with a batch size of $256$ and $2\mathord,000$ training epochs. The learning rate for all datasets is selected from $\curl*{0.01, 0.05, 0.1}$.
We adopted the squared loss as the regression loss ($\sfL = \sfL_2$).
For our single-stage surrogate loss \eqref{eq:sur} and two-stage surrogate loss \eqref{eq:two-stage-sur}, we choose $\ell = \ell_{\rm{log}}$ as the logistic loss. In the experiments, we considered two types of costs: $c_j(x, y) = L(g_j(x), y)$ and $c_j(x, y) = L(g_j(x), y) + \alpha_j$, for $1 \leq j \leq \num$. In the first case, the cost corresponds exactly to the expert's squared error. In the second case, the constant $\alpha_j$ is the base cost for deferring to expert $g_j$. We chose $(\alpha_1, \alpha_2, \alpha_3) = (4.0, 8.0, 12.0)$.
For evaluation, we compute the system mean squared error (MSE), that is the average squared difference between the target value and the prediction made by the predictor $h$ or the expert selected by the deferral function $r$.\ignore{ $\frac{\sum_{i = 1}^n \ldef(h, r, x_i, y_i)}{n}$ and the deferral ratio to the $j$th expert:
$\frac{\sum_{i = 1}^n 1_{\rr(x) = j}}{n}$} We also report the empirical regression loss, $\frac{1}{n}\sum_{i = 1}^n \sfL(h(x_i), y_i)$, of the base model used in the two-stage algorithm.

\textbf{Results.}
In 
Table~\ref{tab:deferral-1}, we report the mean and standard deviation of the empirical regression loss of the base model, as well as the System MSE obtained by using a single expert $g_1$, two experts $g_1$ and $g_2$ and three experts $g_1$, $g_2$ and $g_3$, over three random splits of the dataset.

Table~\ref{tab:deferral-1} shows that the performance of both our single-stage and two-stage algorithms improves as more experts are taken into account across the \texttt{Airfoil}, \texttt{Housing} and \texttt{Concrete} datasets. In particular, our algorithms are able to effectively defer difficult test instances to more suitable experts and outperform the base model. Table~\ref{tab:deferral-1} also shows that the two-stage algorithm usually does not perform as well as the single-stage one in the regression setting, mainly due to the error accumulation in the two-stage process, particularly if the first-stage predictor has large errors. However, the two-stage algorithm is still useful when an existing predictor cannot be retrained due to cost or time constraints. In such cases, we can still improve its performance by learning a multi-expert deferral with our two-stage surrogate losses.


\bibliography{regdef}

\begin{thebibliography}{69}
\providecommand{\natexlab}[1]{#1}
\providecommand{\url}[1]{\texttt{#1}}
\expandafter\ifx\csname urlstyle\endcsname\relax
  \providecommand{\doi}[1]{doi: #1}\else
  \providecommand{\doi}{doi: \begingroup \urlstyle{rm}\Url}\fi

\bibitem[Asuncion and Newman(2007)]{asuncion2007uci}
Arthur Asuncion and David Newman.
\newblock Uci machine learning repository, 2007.

\bibitem[Awasthi et~al.(2022{\natexlab{a}})Awasthi, Mao, Mohri, and
  Zhong]{awasthi2022h}
Pranjal Awasthi, Anqi Mao, Mehryar Mohri, and Yutao Zhong.
\newblock {$H$}-consistency bounds for surrogate loss minimizers.
\newblock In \emph{International Conference on Machine Learning}, pages
  1117--1174, 2022{\natexlab{a}}.

\bibitem[Awasthi et~al.(2022{\natexlab{b}})Awasthi, Mao, Mohri, and
  Zhong]{awasthi2022multi}
Pranjal Awasthi, Anqi Mao, Mehryar Mohri, and Yutao Zhong.
\newblock Multi-class {$ H $}-consistency bounds.
\newblock In \emph{Advances in neural information processing systems}, pages
  782--795, 2022{\natexlab{b}}.

\bibitem[Bartlett and Wegkamp(2008)]{bartlett2008classification}
Peter~L Bartlett and Marten~H Wegkamp.
\newblock Classification with a reject option using a hinge loss.
\newblock \emph{Journal of Machine Learning Research}, 9\penalty0 (8), 2008.

\bibitem[Bartlett et~al.(2006)Bartlett, Jordan, and
  McAuliffe]{bartlett2006convexity}
Peter~L. Bartlett, Michael~I. Jordan, and Jon~D. McAuliffe.
\newblock Convexity, classification, and risk bounds.
\newblock \emph{Journal of the American Statistical Association}, 101\penalty0
  (473):\penalty0 138--156, 2006.

\bibitem[Benz and Rodriguez(2022)]{benz2022counterfactual}
Nina L~Corvelo Benz and Manuel~Gomez Rodriguez.
\newblock Counterfactual inference of second opinions.
\newblock In \emph{Uncertainty in Artificial Intelligence}, pages 453--463.
  PMLR, 2022.

\bibitem[Berkson(1944)]{Berkson1944}
Joseph Berkson.
\newblock Application of the logistic function to bio-assay.
\newblock \emph{Journal of the American Statistical Association}, 39:\penalty0
  357–--365, 1944.

\bibitem[Berkson(1951)]{Berkson1951}
Joseph Berkson.
\newblock Why {I} prefer logits to probits.
\newblock \emph{Biometrics}, 7\penalty0 (4):\penalty0 327–--339, 1951.

\bibitem[Bubeck et~al.(2023)Bubeck, Chandrasekaran, Eldan, Gehrke, Horvitz,
  Kamar, Lee, Lee, Li, Lundberg, et~al.]{bubeck2023sparks}
S{\'e}bastien Bubeck, Varun Chandrasekaran, Ronen Eldan, Johannes Gehrke, Eric
  Horvitz, Ece Kamar, Peter Lee, Yin~Tat Lee, Yuanzhi Li, Scott Lundberg,
  et~al.
\newblock Sparks of artificial general intelligence: Early experiments with
  gpt-4.
\newblock \emph{arXiv preprint arXiv:2303.12712}, 2023.

\bibitem[Cao et~al.(2022)Cao, Cai, Feng, Gu, Gu, An, Niu, and
  Sugiyama]{caogeneralizing}
Yuzhou Cao, Tianchi Cai, Lei Feng, Lihong Gu, Jinjie Gu, Bo~An, Gang Niu, and
  Masashi Sugiyama.
\newblock Generalizing consistent multi-class classification with rejection to
  be compatible with arbitrary losses.
\newblock In \emph{Advances in neural information processing systems}, 2022.

\bibitem[Charoenphakdee et~al.(2021)Charoenphakdee, Cui, Zhang, and
  Sugiyama]{charoenphakdee2021classification}
Nontawat Charoenphakdee, Zhenghang Cui, Yivan Zhang, and Masashi Sugiyama.
\newblock Classification with rejection based on cost-sensitive classification.
\newblock In \emph{International Conference on Machine Learning}, pages
  1507--1517, 2021.

\bibitem[Cheng et~al.(2023)Cheng, Cao, Wang, Wei, An, and
  Feng]{cheng2023regression}
Xin Cheng, Yuzhou Cao, Haobo Wang, Hongxin Wei, Bo~An, and Lei Feng.
\newblock Regression with cost-based rejection.
\newblock In \emph{Advances in Neural Information Processing Systems}, 2023.

\bibitem[Chow(1970)]{chow1970optimum}
C~Chow.
\newblock On optimum recognition error and reject tradeoff.
\newblock \emph{IEEE Transactions on information theory}, 16\penalty0
  (1):\penalty0 41--46, 1970.

\bibitem[Chow(1957)]{Chow1957}
C.K. Chow.
\newblock An optimum character recognition system using decision function.
\newblock \emph{IEEE T. C.}, 1957.

\bibitem[Cortes and Vapnik(1995)]{cortes1995support}
Corinna Cortes and Vladimir Vapnik.
\newblock Support-vector networks.
\newblock \emph{Machine learning}, 20:\penalty0 273--297, 1995.

\bibitem[Cortes et~al.(2016{\natexlab{a}})Cortes, DeSalvo, and
  Mohri]{CortesDeSalvoMohri2016bis}
Corinna Cortes, Giulia DeSalvo, and Mehryar Mohri.
\newblock Boosting with abstention.
\newblock In \emph{Advances in Neural Information Processing Systems}, pages
  1660--1668, 2016{\natexlab{a}}.

\bibitem[Cortes et~al.(2016{\natexlab{b}})Cortes, DeSalvo, and
  Mohri]{cortes2016learning}
Corinna Cortes, Giulia DeSalvo, and Mehryar Mohri.
\newblock Learning with rejection.
\newblock In \emph{International Conference on Algorithmic Learning Theory},
  pages 67--82, 2016{\natexlab{b}}.

\bibitem[Cortes et~al.(2023)Cortes, DeSalvo, and Mohri]{CortesDeSalvoMohri2023}
Corinna Cortes, Giulia DeSalvo, and Mehryar Mohri.
\newblock Theory and algorithms for learning with rejection in binary
  classification.
\newblock \emph{Annals of Mathematics and Artificial Intelligence}, 2023.

\bibitem[Cruz et~al.(2018)Cruz, Sabourin, and Cavalcanti]{cruz2018dynamic}
Rafael~MO Cruz, Robert Sabourin, and George~DC Cavalcanti.
\newblock Dynamic classifier selection: Recent advances and perspectives.
\newblock \emph{Information Fusion}, 41:\penalty0 195--216, 2018.

\bibitem[De et~al.(2020)De, Koley, Ganguly, and
  Gomez-Rodriguez]{de2020regression}
Abir De, Paramita Koley, Niloy Ganguly, and Manuel Gomez-Rodriguez.
\newblock Regression under human assistance.
\newblock In \emph{Proceedings of the AAAI Conference on Artificial
  Intelligence}, pages 2611--2620, 2020.

\bibitem[Ekanayake et~al.(2023)Ekanayake, Zois, and
  Chelmis]{ekanayake2023sequential}
Sachini~Piyoni Ekanayake, Daphney-Stavroula Zois, and Charalampos Chelmis.
\newblock Sequential datum--wise feature acquisition and classifier selection.
\newblock \emph{IEEE Transactions on Artificial Intelligence}, 2023.

\bibitem[El-Yaniv and Wiener(2012)]{el2012active}
Ran El-Yaniv and Yair Wiener.
\newblock Active learning via perfect selective classification.
\newblock \emph{Journal of Machine Learning Research}, 13\penalty0 (2), 2012.

\bibitem[El-Yaniv et~al.(2010)]{el2010foundations}
Ran El-Yaniv et~al.
\newblock On the foundations of noise-free selective classification.
\newblock \emph{Journal of Machine Learning Research}, 11\penalty0 (5), 2010.

\bibitem[Geifman and El-Yaniv(2017)]{geifman2017selective}
Yonatan Geifman and Ran El-Yaniv.
\newblock Selective classification for deep neural networks.
\newblock In \emph{Advances in neural information processing systems}, 2017.

\bibitem[Geifman and El-Yaniv(2019)]{geifman2019selectivenet}
Yonatan Geifman and Ran El-Yaniv.
\newblock Selectivenet: A deep neural network with an integrated reject option.
\newblock In \emph{International conference on machine learning}, pages
  2151--2159, 2019.

\bibitem[Hemmer et~al.(2022)Hemmer, Schellhammer, V{\"o}ssing, Jakubik, and
  Satzger]{hemmer2022forming}
Patrick Hemmer, Sebastian Schellhammer, Michael V{\"o}ssing, Johannes Jakubik,
  and Gerhard Satzger.
\newblock Forming effective human-ai teams: Building machine learning models
  that complement the capabilities of multiple experts.
\newblock \emph{arXiv preprint arXiv:2206.07948}, 2022.

\bibitem[Jiang et~al.(2020)Jiang, Zhao, and Wang]{jiang2020risk}
Wenming Jiang, Ying Zhao, and Zehan Wang.
\newblock Risk-controlled selective prediction for regression deep neural
  network models.
\newblock In \emph{2020 International Joint Conference on Neural Networks
  (IJCNN)}, pages 1--8, 2020.

\bibitem[Kerrigan et~al.(2021)Kerrigan, Smyth, and
  Steyvers]{kerrigan2021combining}
Gavin Kerrigan, Padhraic Smyth, and Mark Steyvers.
\newblock Combining human predictions with model probabilities via confusion
  matrices and calibration.
\newblock \emph{Advances in Neural Information Processing Systems},
  34:\penalty0 4421--4434, 2021.

\bibitem[Keswani et~al.(2021)Keswani, Lease, and
  Kenthapadi]{keswani2021towards}
Vijay Keswani, Matthew Lease, and Krishnaram Kenthapadi.
\newblock Towards unbiased and accurate deferral to multiple experts.
\newblock In \emph{Proceedings of the 2021 AAAI/ACM Conference on AI, Ethics,
  and Society}, pages 154--165, 2021.

\bibitem[Kingma and Ba(2014)]{kingma2014adam}
Diederik~P Kingma and Jimmy Ba.
\newblock Adam: A method for stochastic optimization.
\newblock \emph{arXiv preprint arXiv:1412.6980}, 2014.

\bibitem[Ko et~al.(2008)Ko, Sabourin, and Britto~Jr]{ko2008dynamic}
Albert~HR Ko, Robert Sabourin, and Alceu~Souza Britto~Jr.
\newblock From dynamic classifier selection to dynamic ensemble selection.
\newblock \emph{Pattern recognition}, 41\penalty0 (5):\penalty0 1718--1731,
  2008.

\bibitem[Li et~al.(2023)Li, Liu, Sun, and Wang]{li2023no}
Xiaocheng Li, Shang Liu, Chunlin Sun, and Hanzhao Wang.
\newblock When no-rejection learning is optimal for regression with rejection.
\newblock \emph{arXiv preprint arXiv:2307.02932}, 2023.

\bibitem[Madras et~al.(2018)Madras, Pitassi, and Zemel]{madras2018predict}
David Madras, Toni Pitassi, and Richard Zemel.
\newblock Predict responsibly: improving fairness and accuracy by learning to
  defer.
\newblock In \emph{Advances in Neural Information Processing Systems}, 2018.

\bibitem[Mao et~al.(2023{\natexlab{a}})Mao, Mohri, Mohri, and
  Zhong]{MaoMohriZhong2023}
Anqi Mao, Christopher Mohri, Mehryar Mohri, and Yutao Zhong.
\newblock Two-stage learning to defer with multiple experts.
\newblock In \emph{Advances in Neural Information Processing Systems (NeurIPS
  2023)}, New Orleans, Louisiana, 2023{\natexlab{a}}. MIT Press.

\bibitem[Mao et~al.(2023{\natexlab{b}})Mao, Mohri, and
  Zhong]{MaoMohriZhong2023characterization}
Anqi Mao, Mehryar Mohri, and Yutao Zhong.
\newblock {H}-consistency bounds: Characterization and extensions.
\newblock In \emph{Advances in Neural Information Processing Systems},
  2023{\natexlab{b}}.

\bibitem[Mao et~al.(2023{\natexlab{c}})Mao, Mohri, and
  Zhong]{MaoMohriZhong2023ranking}
Anqi Mao, Mehryar Mohri, and Yutao Zhong.
\newblock {H}-consistency bounds for pairwise misranking loss surrogates.
\newblock In \emph{International conference on Machine learning},
  2023{\natexlab{c}}.

\bibitem[Mao et~al.(2023{\natexlab{d}})Mao, Mohri, and
  Zhong]{MaoMohriZhong2023rankingabs}
Anqi Mao, Mehryar Mohri, and Yutao Zhong.
\newblock Ranking with abstention.
\newblock In \emph{ICML 2023 Workshop The Many Facets of Preference-Based
  Learning}, 2023{\natexlab{d}}.

\bibitem[Mao et~al.(2023{\natexlab{e}})Mao, Mohri, and
  Zhong]{MaoMohriZhong2023structured}
Anqi Mao, Mehryar Mohri, and Yutao Zhong.
\newblock Structured prediction with stronger consistency guarantees.
\newblock In \emph{Advances in Neural Information Processing Systems},
  2023{\natexlab{e}}.

\bibitem[Mao et~al.(2023{\natexlab{f}})Mao, Mohri, and Zhong]{mao2023cross}
Anqi Mao, Mehryar Mohri, and Yutao Zhong.
\newblock Cross-entropy loss functions: Theoretical analysis and applications.
\newblock In \emph{International Conference on Machine Learning},
  2023{\natexlab{f}}.

\bibitem[Mao et~al.(2024{\natexlab{a}})Mao, Mohri, and
  Zhong]{MaoMohriZhong2024}
Anqi Mao, Mehryar Mohri, and Yutao Zhong.
\newblock Principled approaches for learning to defer with multiple experts.
\newblock In \emph{International Symposium on Artificial Intelligence and
  Mathematics}, 2024{\natexlab{a}}.

\bibitem[Mao et~al.(2024{\natexlab{b}})Mao, Mohri, and
  Zhong]{MaoMohriZhong2024predictor}
Anqi Mao, Mehryar Mohri, and Yutao Zhong.
\newblock Predictor-rejector multi-class abstention: Theoretical analysis and
  algorithms.
\newblock In \emph{Algorithmic Learning Theory}, 2024{\natexlab{b}}.

\bibitem[Mao et~al.(2024{\natexlab{c}})Mao, Mohri, and
  Zhong]{MaoMohriZhong2024score}
Anqi Mao, Mehryar Mohri, and Yutao Zhong.
\newblock Theoretically grounded loss functions and algorithms for score-based
  multi-class abstention.
\newblock In \emph{International Conference on Artificial Intelligence and
  Statistics}, 2024{\natexlab{c}}.

\bibitem[Mohri et~al.(2024)Mohri, Andor, Choi, Collins, Mao, and
  Zhong]{MohriAndorChoiCollinsMaoZhong2024learning}
Christopher Mohri, Daniel Andor, Eunsol Choi, Michael Collins, Anqi Mao, and
  Yutao Zhong.
\newblock Learning to reject with a fixed predictor: Application to
  decontextualization.
\newblock In \emph{International Conference on Learning Representations}, 2024.

\bibitem[Mohri et~al.(2018)Mohri, Rostamizadeh, and
  Talwalkar]{mohri2018foundations}
Mehryar Mohri, Afshin Rostamizadeh, and Ameet Talwalkar.
\newblock \emph{Foundations of machine learning}.
\newblock MIT press, 2018.

\bibitem[Mozannar and Sontag(2020)]{mozannar2020consistent}
Hussein Mozannar and David Sontag.
\newblock Consistent estimators for learning to defer to an expert.
\newblock In \emph{International Conference on Machine Learning}, pages
  7076--7087, 2020.

\bibitem[Nair and Hinton(2010)]{nair2010rectified}
Vinod Nair and Geoffrey~E Hinton.
\newblock Rectified linear units improve restricted boltzmann machines.
\newblock In \emph{Proceedings of the 27th international conference on machine
  learning (ICML-10)}, pages 807--814, 2010.

\bibitem[Narasimhan et~al.(2022)Narasimhan, Jitkrittum, Menon, Rawat, and
  Kumar]{narasimhanpost}
Harikrishna Narasimhan, Wittawat Jitkrittum, Aditya~Krishna Menon, Ankit~Singh
  Rawat, and Sanjiv Kumar.
\newblock Post-hoc estimators for learning to defer to an expert.
\newblock In \emph{Advances in Neural Information Processing Systems}, pages
  29292--29304, 2022.

\bibitem[Ni et~al.(2019)Ni, Charoenphakdee, Honda, and Sugiyama]{NiCHS19}
Chenri Ni, Nontawat Charoenphakdee, Junya Honda, and Masashi Sugiyama.
\newblock On the calibration of multiclass classification with rejection.
\newblock In \emph{Advances in Neural Information Processing Systems}, pages
  2582--2592, 2019.

\bibitem[Okati et~al.(2021)Okati, De, and Rodriguez]{okati2021differentiable}
Nastaran Okati, Abir De, and Manuel Rodriguez.
\newblock Differentiable learning under triage.
\newblock \emph{Advances in Neural Information Processing Systems},
  34:\penalty0 9140--9151, 2021.

\bibitem[Ramaswamy et~al.(2018)Ramaswamy, Tewari, and
  Agarwal]{ramaswamy2018consistent}
Harish~G Ramaswamy, Ambuj Tewari, and Shivani Agarwal.
\newblock Consistent algorithms for multiclass classification with an abstain
  option.
\newblock \emph{Electronic Journal of Statistics}, 12\penalty0 (1):\penalty0
  530--554, 2018.

\bibitem[Shah et~al.(2022)Shah, Bu, Lee, Das, Panda, Sattigeri, and
  Wornell]{shah2022selective}
Abhin Shah, Yuheng Bu, Joshua~K Lee, Subhro Das, Rameswar Panda, Prasanna
  Sattigeri, and Gregory~W Wornell.
\newblock Selective regression under fairness criteria.
\newblock In \emph{International Conference on Machine Learning}, pages
  19598--19615, 2022.

\bibitem[Steinwart(2007)]{steinwart2007compare}
Ingo Steinwart.
\newblock How to compare different loss functions and their risks.
\newblock \emph{Constructive Approximation}, 26\penalty0 (2):\penalty0
  225--287, 2007.

\bibitem[Straitouri et~al.(2022)Straitouri, Wang, Okati, and
  Rodriguez]{straitouri2022provably}
Eleni Straitouri, Lequn Wang, Nastaran Okati, and Manuel~Gomez Rodriguez.
\newblock Provably improving expert predictions with conformal prediction.
\newblock \emph{arXiv preprint arXiv:2201.12006}, 2022.

\bibitem[Tailor et~al.(2024)Tailor, Patra, Verma, Manggala, and
  Nalisnick]{tailor2024learning}
Dharmesh Tailor, Aditya Patra, Rajeev Verma, Putra Manggala, and Eric
  Nalisnick.
\newblock Learning to defer to a population: A meta-learning approach.
\newblock \emph{arXiv preprint arXiv:2403.02683}, 2024.

\bibitem[Tewari and Bartlett(2007)]{tewari2007consistency}
Ambuj Tewari and Peter~L. Bartlett.
\newblock On the consistency of multiclass classification methods.
\newblock \emph{Journal of Machine Learning Research}, 8\penalty0
  (36):\penalty0 1007--1025, 2007.

\bibitem[Verhulst(1838)]{Verhulst1838}
Pierre~François Verhulst.
\newblock Notice sur la loi que la population suit dans son accroissement.
\newblock \emph{Correspondance math\'ematique et physique}, 10:\penalty0
  113–--121, 1838.

\bibitem[Verhulst(1845)]{Verhulst1845}
Pierre~François Verhulst.
\newblock Recherches math\'ematiques sur la loi d'accroissement de la
  population.
\newblock \emph{Nouveaux M\'emoires de l'Acad\'emie Royale des Sciences et
  Belles-Lettres de Bruxelles}, 18:\penalty0 1–--42, 1845.

\bibitem[Verma and Nalisnick(2022)]{verma2022calibrated}
Rajeev Verma and Eric Nalisnick.
\newblock Calibrated learning to defer with one-vs-all classifiers.
\newblock In \emph{International Conference on Machine Learning}, pages
  22184--22202, 2022.

\bibitem[Verma et~al.(2023)Verma, Barrej{\'o}n, and
  Nalisnick]{verma2023learning}
Rajeev Verma, Daniel Barrej{\'o}n, and Eric Nalisnick.
\newblock Learning to defer to multiple experts: Consistent surrogate losses,
  confidence calibration, and conformal ensembles.
\newblock In \emph{International Conference on Artificial Intelligence and
  Statistics}, pages 11415--11434, 2023.

\bibitem[Wei et~al.(2022)Wei, Tay, Bommasani, Raffel, Zoph, Borgeaud, Yogatama,
  Bosma, Zhou, Metzler, Chi, Hashimoto, Vinyals, Liang, Dean, and
  Fedus]{WeiEtAl2022}
Jason Wei, Yi~Tay, Rishi Bommasani, Colin Raffel, Barret Zoph, Sebastian
  Borgeaud, Dani Yogatama, Maarten Bosma, Denny Zhou, Donald Metzler, Ed~H.
  Chi, Tatsunori Hashimoto, Oriol Vinyals, Percy Liang, Jeff Dean, and William
  Fedus.
\newblock Emergent abilities of large language models.
\newblock \emph{CoRR}, abs/2206.07682, 2022.

\bibitem[Wiener and El-Yaniv(2011)]{wiener2011agnostic}
Yair Wiener and Ran El-Yaniv.
\newblock Agnostic selective classification.
\newblock In \emph{Advances in neural information processing systems}, 2011.

\bibitem[Wiener and El-Yaniv(2012)]{wiener2012pointwise}
Yair Wiener and Ran El-Yaniv.
\newblock Pointwise tracking the optimal regression function.
\newblock \emph{Advances in Neural Information Processing Systems}, 25, 2012.

\bibitem[Wiener and El-Yaniv(2015)]{wiener2015agnostic}
Yair Wiener and Ran El-Yaniv.
\newblock Agnostic pointwise-competitive selective classification.
\newblock \emph{Journal of Artificial Intelligence Research}, 52:\penalty0
  171--201, 2015.

\bibitem[Yuan and Wegkamp(2010)]{yuan2010classification}
Ming Yuan and Marten Wegkamp.
\newblock Classification methods with reject option based on convex risk
  minimization.
\newblock \emph{Journal of Machine Learning Research}, 11\penalty0 (1), 2010.

\bibitem[Yuan and Wegkamp(2011)]{WegkampYuan2011}
Ming Yuan and Marten Wegkamp.
\newblock {SVM}s with a reject option.
\newblock In \emph{Bernoulli}, 2011.

\bibitem[Zaoui et~al.(2020)Zaoui, Denis, and Hebiri]{zaoui2020regression}
Ahmed Zaoui, Christophe Denis, and Mohamed Hebiri.
\newblock Regression with reject option and application to knn.
\newblock In \emph{Advances in Neural Information Processing Systems}, pages
  20073--20082, 2020.

\bibitem[Zhang(2004{\natexlab{a}})]{Zhang2003}
Tong Zhang.
\newblock Statistical behavior and consistency of classification methods based
  on convex risk minimization.
\newblock \emph{The Annals of Statistics}, 32\penalty0 (1):\penalty0 56--85,
  2004{\natexlab{a}}.

\bibitem[Zhang(2004{\natexlab{b}})]{zhang2004statistical}
Tong Zhang.
\newblock Statistical analysis of some multi-category large margin
  classification methods.
\newblock \emph{Journal of Machine Learning Research}, 5\penalty0
  (Oct):\penalty0 1225--1251, 2004{\natexlab{b}}.

\bibitem[Zheng et~al.(2023)Zheng, Wu, Bao, Cao, Li, and
  Zhu]{zheng2023revisiting}
Chenyu Zheng, Guoqiang Wu, Fan Bao, Yue Cao, Chongxuan Li, and Jun Zhu.
\newblock Revisiting discriminative vs. generative classifiers: Theory and
  implications.
\newblock \emph{arXiv preprint arXiv:2302.02334}, 2023.

\end{thebibliography}

\newpage
\appendix
\onecolumn

\renewcommand{\contentsname}{Contents of Appendix}
\tableofcontents
\addtocontents{toc}{\protect\setcounter{tocdepth}{4}} 
\clearpage

\section{Useful lemmas}

\Ldef*

\begin{proof}
  Observe that, for any $x \in \sX$, since $\rr(x) = 0$ if and only if
  $\rr(x) \neq j$ for all $j \geq 1$, the following equality holds:
\[
1_{\rr(x) = 0}
= 1_{\bigwedge_{j = 1}^{n_e} \curl*{\rr(x) \neq j}}
= \sum_{j = 1}^{n_e} 1_{\rr(x) \neq j} - (n_e - 1).
\]
Similarly, since $\rr(x) = j$ if and only if
  $\rr(x) \neq k$ for $k \neq j$ and $\rr(x) \neq 0$, the following equality holds:
\[
1_{\rr(x) = j} = 1_{\rr(x) \neq 0} + \sum_{k = 1}^{n_e} 1_{\rr(x) \neq k} 1_{k \neq j} - (n_e - 1).
\]
In view of these identities, starting from the definition of $\ldef$, we
can write:
\begin{align*}
  & \ldef(h, r, x, y)\\
  & = \sfL(h(x), y) 1_{\rr(x) = 0} + \sum_{j = 1}^{\num} c_j(x, y) 1_{\rr(x) = j}\\
  & = \sfL(h(x), y) \bracket*{\sum_{j = 1}^{n_e} 1_{\rr(x) \neq j} - (n_e - 1)}
  + \sum_{j = 1}^{\num} c_j(x, y) \bracket*{1_{\rr(x) \neq 0} + \sum_{k = 1}^{n_e} 1_{\rr(x) \neq k} 1_{k \neq j} - (n_e - 1)}\\
  & = \bracket*{\sum_{j = 1}^{\num} c_j(x,y)} 1_{\rr(x) \neq 0}
  + \sum_{j = 1}^{\num} \sfL(h(x), y) 1_{\rr(x) \neq j}
  + \sum_{j = 1}^{\num} \sum_{k = 1}^{\num} c_j(x, y) 1_{k \neq j} 1_{\rr(x) \neq k}\\
  & \quad - \paren*{\num - 1} \bracket*{\sfL(h(x), y) + \sum_{j = 1}^{\num} c_j(x, y)}\\
  & = \bracket*{\sum_{j = 1}^{\num} c_j(x,y)} 1_{\rr(x) \neq 0}
  + \sum_{j = 1}^{\num} \sfL(h(x), y) 1_{\rr(x) \neq j}
  + \sum_{k = 1}^{\num} \sum_{j = 1}^{\num} c_k(x, y) 1_{k \neq j} 1_{\rr(x) \neq j}\\
  & \quad - \paren*{\num - 1} \bracket*{\sfL(h(x), y) + \sum_{j = 1}^{\num} c_j(x, y)}
    \tag{change of variables $k$ and $j$}\\
  & = \bracket*{\sum_{j = 1}^{\num} c_j(x,y)} 1_{\rr(x) \neq 0} 
  + \sum_{j = 1}^{\num} \bracket*{\sfL(h(x), y)
    + \sum_{k = 1}^{\num} c_k(x, y) 1_{k \neq j}} 1_{\rr(x) \neq j}
  - \paren*{\num - 1} \bracket*{\sfL(h(x), y) + \sum_{j = 1}^{\num} c_j(x, y)}.
\end{align*}
This completes the proof.
\end{proof}

\begin{lemma}
\label{lemma:aux}
Assume that the following $\sR$-consistency bound holds for all $r \in \sR$ and any distribution,
\begin{equation*}
\sE_{\ell_{0-1}}(r) - \sE^*_{\ell_{0-1}}(\sR) + \sM_{\ell_{0-1}}(\sR) \leq \Gamma\paren*{\sE_{\ell}(r) - \sE^*_{\ell}(\sR) + \sM_{\ell}(\sR)}.
\end{equation*}
Then, for any $p = (p_0, \ldots, p_{\num})\in \Delta^{\num}$ and $x \in \sX$, we have
\begin{align*}
\sum_{j = 0}^{\num} p_j 1_{\rr(x) \neq j} - \inf_{r \in \sR} \paren*{\sum_{j = 0}^{\num} p_j 1_{\rr(x) \neq j}} \leq \Gamma\paren*{\sum_{j = 0}^{\num} p_j \ell(r, x, j) - \inf_{r \in \sR} \paren*{\sum_{j = 0}^{\num} p_j \ell(r, x, j) }}.
\end{align*}
\end{lemma}
\begin{proof}
For any $x \in \sX$, consider a distribution $\delta_{x}$ that concentrates on that point. Let $p_j = \mathbb{P}(y = j \mid x)$, $j \in [\num]$. Then, by definition, $\sE_{\ell_{0-1}}(r) - \sE^*_{\ell_{0-1}}(\sR) + \sM_{\ell_{0-1}}(\sR)$ can be expressed as 
\begin{equation*}
\sE_{\ell_{0-1}}(r) - \sE^*_{\ell_{0-1}}(\sR) + \sM_{\ell_{0-1}}(\sR) = \sum_{j = 0}^{\num} p_j 1_{\rr(x) \neq j} - \inf_{r \in \sR} \paren*{\sum_{j = 0}^{\num} p_j 1_{\rr(x) \neq j}}.
\end{equation*}
Similarly, $\sE_{\ell}(r) - \sE^*_{\ell}(\sR) + \sM_{\ell}(\sR)$ can be expressed as
\begin{equation*}
\sE_{\ell}(r) - \sE^*_{\ell}(\sR) + \sM_{\ell}(\sR) = \sum_{j = 0}^{\num} p_j \ell(r, x, j) - \inf_{r \in \sR} \paren*{\sum_{j = 0}^{\num} p_j \ell(r, x, j) }.
\end{equation*}
Since the $\sR$-consistency bound holds by the assumption, we complete the proof.
\end{proof}

\section{Proof of Theorem~\ref{thm:single}}
\label{app:exp}

\Single*
\begin{proof}
The conditional error of the deferral loss can be expressed as
\begin{equation}
\label{eq:cond-error}
\begin{aligned}
\E_{y | x}\bracket*{\ldef(h, r, x, y)} =
\E_{y | x}\bracket*{\sfL(h(x), y)} 1_{\rr(x) = 0} + \sum_{j = 1}^{\num} \E_{y | x}\bracket*{c_j(x,y)} 1_{\rr(x) = j}.
\end{aligned}
\end{equation}
Let $\ov c_0(x) = \inf_{h\in \sH}\E_{y | x}\bracket*{\sfL(h(x), y)}$ and $\ov c_j(x) = \E_{y | x}\bracket*{c(x, y)}$.
Thus, the best-in class conditional error of the deferral loss can be expressed as
\begin{equation}
\label{eq:best-cond-error}
\inf_{h\in \sH, r\in \sR}\E_{y | x}\bracket*{\ldef(h, r, x, y)} =  \min_{j \in [\num]} \ov c_j(x).
\end{equation}
The conditional error of the surrogate loss can be expressed as
\begin{equation}
\label{eq:cond-error-sur}
\begin{aligned}
\E_{y | x}\bracket*{\ell_{\ell}(h, r, x, y)} & =
\paren*{\sum_{j = 1}^{\num} \E_{y | x}\bracket*{c_j(x,y)}} \ell(r, x, 0)
+ \sum_{j = 1}^{\num} \paren*{\E_{y | x}\bracket*{\sfL(h(x), y)} + \sum_{j' \neq j}^{\num} \E_{y | x}\bracket*{c_{j'}(x,y)}} \ell(r, x, j)\\
& \quad - \paren*{\num - 1} \E_{y | x}\bracket*{\sfL(h(x), y)}.
\end{aligned}
\end{equation}
Note that the coefficient of term $\E_{y | x}\bracket*{\sfL(h(x), y)}$ satisfies $\sum_{j = 1}^{\num}  \ell(r, x, j) - \paren*{\num - 1} \geq 0$ since $\ell \geq \ell_{0-1}$.
Thus, the best-in class conditional error of the surrogate loss can be expressed as
\begin{equation}
\label{eq:best-cond-error-sur}
\begin{aligned}
& \inf_{h\in \sH,r\in \sR}\E_{y | x}\bracket*{L_{\ell}(h, r, x, y)}\\
& = \inf_{r \in \sR}\bracket*{\paren*{\sum_{j = 1}^{\num} \ov c_j(x)} \ell(r, x, 0)
+ \sum_{j = 1}^{\num} \paren*{\ov c_0(x) + \sum_{j' \neq j}^{\num} \ov c_{j'}(x)} \ell(r, x, j)} - \paren*{\num - 1} \ov c_0(x).
\end{aligned}
\end{equation}
Next, we analyze four cases separately to show that the calibration gap of the surrogate loss can be lower bounded by that of the deferral loss.
\paragraph{Case I: \texorpdfstring{$\rr(x) = 0$}{I} and \texorpdfstring{$\ov c_0(x) \leq \min_{j = 1}^{\num} \ov c_j(x)$}{I}.} In this case, by \eqref{eq:cond-error} and \eqref{eq:best-cond-error}, the calibration gap of the deferral loss can be expressed as
\begin{align*}
\E_{y | x}\bracket*{\ldef(h, r, x, y)}-\inf_{h\in \sH, r\in \sR}\E_{y | x}\bracket*{\ldef(h, r, x, y)} = \E_{y | x}\bracket*{\sfL(h(x), y)} - \inf_{h\in \sH}\E_{y | x}\bracket*{\sfL(h(x), y)}.
\end{align*}
By \eqref{eq:cond-error-sur} and \eqref{eq:best-cond-error-sur}, the calibration gap of the surrogate loss can be expressed as
\begin{align*}
&\E_{y | x}\bracket*{L_{\ell}(h, r, x, y)} - \inf_{h\in \sH, r\in \sR}\E_{y | x}\bracket*{L_{\ell}(h, r, x, y)}\\
& =  \paren*{\sum_{j = 1}^{\num} \ov c_j(x)} \ell(r, x, 0)
+ \sum_{j = 1}^{\num} \paren*{\E_{y | x}\bracket*{\sfL(h(x), y)} + \sum_{j' \neq j}^{\num} \ov c_{j'}(x)} \ell(r, x, j) - \paren*{\num - 1} \E_{y | x} \sfL(h(x), y)\\
&\quad - \inf_{r \in \sR}\bracket*{\paren*{\sum_{j = 1}^{\num} \ov c_j(x)} \ell(r, x, 0)
+ \sum_{j = 1}^{\num} \paren*{\ov c_0(x) + \sum_{j' \neq j}^{\num} \ov c_{j'}(x)} \ell(r, x, j)} + \paren*{\num - 1} \ov c_0(x).
\end{align*}
Since $\ell \geq \ell_{0-1}$, we have $\ell(r, x, j) \geq 1$ for $j \neq 0$. By eliminating the infimum over $\sR$ from the final line, and consequently canceling the terms related to $\ov c_j(x)$ for $j \neq 0$, the calibration gap of the surrogate loss can be lower bounded as
\begin{align*}
&\E_{y | x}\bracket*{L_{\ell}(h, r, x, y)} - \inf_{h\in \sH,r\in \sR}\E_{y | x}\bracket*{L_{\ell}(h, r, x, y)}\\
&\geq \paren*{\E_{y | x}\bracket*{\sfL(h(x), y)} - \inf_{h\in \sH}\E_{y | x}\bracket*{\sfL(h(x), y)}}\paren*{\sum_{j = 1}^{\num} \ell(r, x, j) - \num + 1}\\
&\geq \E_{y | x}\bracket*{\sfL(h(x), y)}- \inf_{h\in \sH}\E_{y | x}\bracket*{\sfL(h(x), y)} \tag{$\sum_{j = 1}^{\num} \ell(r, x, j) - \num + 1 \geq 1$}\\
& = \E_{y | x}\bracket*{\ldef(h, r, x, y)}-\inf_{h\in \sH,r\in \sR}\E_{y | x}\bracket*{\ldef(h, r, x, y)}.
\end{align*}

\paragraph{Case II: \texorpdfstring{$\ov c_0(x) > \min_{j = 1}^{\num} \ov c_j(x)$}{II}.} In this case, by \eqref{eq:cond-error} and \eqref{eq:best-cond-error},  the calibration gap of the deferral loss can be expressed as
\begin{align*}
\E_{y | x}\bracket*{\ldef(h, r, x, y)}-\inf_{h\in \sH,r\in \sR}\E_{y | x}\bracket*{\ldef(h, r, x, y)} = \ov c_{\rr(x)}(x) - \min_{j = 1}^{\num} \ov c_j(x).
\end{align*}
By \eqref{eq:cond-error-sur} and \eqref{eq:best-cond-error-sur}, the calibration gap of the surrogate loss can be expressed as
\begin{align*}
&\E_{y | x}\bracket*{L_{\ell}(h, r, x, y)} - \inf_{h\in \sH, r\in \sR}\E_{y | x}\bracket*{L_{\ell}(h, r, x, y)}\\
& =  \paren*{\sum_{j = 1}^{\num} \ov c_j(x)} \ell(r, x, 0)
+ \sum_{j = 1}^{\num} \paren*{\E_{y | x}\bracket*{\sfL(h(x), y)} + \sum_{j' \neq j}^{\num} \ov c_{j'}(x)} \ell(r, x, j) - \paren*{\num - 1} \E_{y | x}\bracket*{\sfL(h(x), y)}\\
&\quad - \inf_{r \in \sR}\bracket*{\paren*{\sum_{j = 1}^{\num} \ov c_j(x)} \ell(r, x, 0)
+ \sum_{j = 1}^{\num} \paren*{\ov c_0(x) + \sum_{j' \neq j}^{\num} \ov c_{j'}(x)} \ell(r, x, j)} + \paren*{\num - 1} \ov c_0(x).
\end{align*}
Using the fact that $\ov c_0(x) = \inf_{h\in \sH}\E_{y | x}\bracket*{\sfL(h(x), y)} \leq \E_{y | x}\bracket*{\sfL(h(x), y)}$, the calibration gap of the surrogate loss can be lower bounded as
\begin{align*}
& \E_{y | x}\bracket*{L_{\ell}(h, r, x, y)} - \inf_{h\in \sH,r\in \sR}\E_{y | x}\bracket*{L_{\ell}(h, r, x, y)}\\
& \geq \paren*{\sum_{j = 1}^{\num} \ov c_j(x)} \ell(r, x, 0)
+ \sum_{j = 1}^{\num} \paren*{\E_{y | x}\bracket*{\sfL(h(x), y)} + \sum_{j' \neq j}^{\num} \ov c_{j'}(x)} \ell(r, x, j)\\
&\quad - \inf_{r \in \sR}\bracket*{\paren*{\sum_{j = 1}^{\num} \ov c_j(x)} \ell(r, x, 0)
+ \sum_{j = 1}^{\num} \paren*{\E_{y | x}\bracket*{\sfL(h(x), y)} + \sum_{j' \neq j}^{\num} \ov c_{j'}(x)} \ell(r, x, j)}\\
& = \num \paren*{\E_{y | x}\bracket*{\sfL(h(x), y)} + \sum_{j = 1 }^{\num} \ov c_{j}(x)} \bracket*{\sum_{j = 0}^{\num} p_j \ell(r, x, j) - \inf_{r \in \sR} \paren*{\sum_{j = 0}^{\num} p_j \ell(r, x, j) }}
\end{align*}
where we let $p_0 = \frac{\sum_{j = 1}^{\num} \ov c_j(x)}{\num \paren*{\E_{y | x}\bracket*{\sfL(h(x), y)} + \sum_{j = 1 }^{\num} \ov c_{j}(x)}}$ and $p_j = \frac{\E_{y | x}\bracket*{\sfL(h(x), y)} + \sum_{j' \neq j}^{\num} \ov c_{j'}(x)}{\num \paren*{\E_{y | x}\bracket*{\sfL(h(x), y)} + \sum_{j = 1 }^{\num} \ov c_{j}(x)}}$, $j = \curl*{1, \ldots, \num}$ in the last equality. 
By Lemma~\ref{lemma:aux}, we have
\begin{align*}
& \sum_{j = 0}^{\num} p_j \ell(r, x, j) - \inf_{r \in \sR} \paren*{\sum_{j = 0}^{\num} p_j \ell(r, x, j) }\\
& \geq \Gamma^{-1}\paren*{\sum_{j = 0}^{\num} p_j 1_{\rr(x) \neq j} - \inf_{r \in \sR} \paren*{\sum_{j = 0}^{\num} p_j 1_{\rr(x) \neq j}}} \\
& = \Gamma^{-1}\paren*{\max_{j \in [\num]}p_j - p_{\rr(x)}}\\
& = \Gamma^{-1}\paren*{\frac{\E_{y | x}\bracket*{\sfL(h(x), y)} - \min_{j = 1}^{\num} \ov c_j(x)}{\num \paren*{\E_{y | x}\bracket*{\sfL(h(x), y)} + \sum_{j = 1 }^{\num} \ov c_{j}(x)}}}
\end{align*}
Therefore, we obtain
\begin{align*}
& \E_{y | x}\bracket*{L_{\ell}(h, r, x, y)} - \inf_{h\in \sH,r\in \sR}\E_{y | x}\bracket*{L_{\ell}(h, r, x, y)}\\
& \geq  \num \paren*{\E_{y | x}\bracket*{\sfL(h(x), y)} + \sum_{j = 1 }^{\num} \ov c_{j}(x)}  \Gamma^{-1}\paren*{\frac{\E_{y | x}\bracket*{\sfL(h(x), y)} - \min_{j = 1}^{\num} \ov c_j(x)}{\num \paren*{\E_{y | x}\bracket*{\sfL(h(x), y)} + \sum_{j = 1 }^{\num} \ov c_{j}(x)}}} \\
& \geq
\num \paren*{\E_{y | x}\bracket*{\sfL(h(x), y)} + \sum_{j = 1 }^{\num} \ov c_{j}(x)}  \Gamma^{-1}\paren*{\frac{\E_{y | x}\bracket*{\ldef(h, r, x, y)}-\inf_{h\in \sH,r\in \sR}\E_{y | x}\bracket*{\ldef(h, r, x, y)}}{\num \paren*{\E_{y | x}\bracket*{\sfL(h(x), y)} + \sum_{j = 1 }^{\num} \ov c_{j}(x)}}}\\
& \geq \frac{1}{\beta^{\frac{1}{\alpha}}} \frac{\paren*{\E_{y | x}\bracket*{\ldef(h, r, x, y)}-\inf_{h\in \sH,r\in \sR}\E_{y | x}\bracket*{\ldef(h, r, x, y)}}^{\frac{1}{\alpha}}}{\paren*{\num(\ul + \sum_{j = 1}^{\num}\uc_j)}^{\frac{1}{\alpha} - 1}}
\end{align*}
where we use the fact that $\Gamma(t) = \beta t^{\alpha}$, $\alpha \in (0, 1]$, $\beta > 0$, $\sfL \leq \ul$ and $c_j \leq \uc_j$, $j = \curl*{1, \ldots, \num}$ in the last inequality.

\paragraph{Case III: \texorpdfstring{$\rr(x) > 0$}{III} and \texorpdfstring{$\ov c_0(x) \leq \min_{j = 1}^{\num} \ov c_j(x)$}{III}.} In this case, by \eqref{eq:cond-error} and \eqref{eq:best-cond-error},  the calibration gap of the deferral loss can be expressed as
\begin{align*}
\E_{y | x}\bracket*{\ldef(h, r, x, y)}-\inf_{h\in \sH,r\in \sR}\E_{y | x}\bracket*{\ldef(h, r, x, y)} = \ov c_{\rr(x)}(x) - \ov c_0(x).
\end{align*}
By \eqref{eq:cond-error-sur} and \eqref{eq:best-cond-error-sur}, the calibration gap of the surrogate loss can be expressed as
\begin{align*}
&\E_{y | x}\bracket*{L_{\ell}(h, r, x, y)} - \inf_{h\in \sH, r\in \sR}\E_{y | x}\bracket*{L_{\ell}(h, r, x, y)}\\
& =  \paren*{\sum_{j = 1}^{\num} \ov c_j(x)} \ell(r, x, 0)
+ \sum_{j = 1}^{\num} \paren*{\E_{y | x}\bracket*{\sfL(h(x), y)} + \sum_{j' \neq j}^{\num} \ov c_{j'}(x)} \ell(r, x, j) - \paren*{\num - 1} \E_{y | x}\bracket*{\sfL(h(x), y)}\\
&\quad - \inf_{r \in \sR}\bracket*{\paren*{\sum_{j = 1}^{\num} \ov c_j(x)} \ell(r, x, 0)
+ \sum_{j = 1}^{\num} \paren*{\ov c_0(x) + \sum_{j' \neq j}^{\num} \ov c_{j'}(x)} \ell(r, x, j)} + \paren*{\num - 1} \ov c_0(x).
\end{align*}
Using the fact that $\E_{y | x}\bracket*{\sfL(h(x), y)} \geq \inf_{h\in \sH}\E_{y | x}\bracket*{\sfL(h(x), y)} = \ov c_0(x)$, the calibration gap of the surrogate loss can be lower bounded as
\begin{align*}
& \E_{y | x}\bracket*{L_{\ell}(h, r, x, y)} - \inf_{h\in \sH,r\in \sR}\E_{y | x}\bracket*{L_{\ell}(h, r, x, y)}\\
& \geq \paren*{\sum_{j = 1}^{\num} \ov c_j(x)} \ell(r, x, 0)
+ \sum_{j = 1}^{\num} \paren*{\ov c_0(x) + \sum_{j' \neq j}^{\num} \ov c_{j'}(x)} \ell(r, x, j)\\
&\quad - \inf_{r \in \sR}\bracket*{\paren*{\sum_{j = 1}^{\num} \ov c_j(x)} \ell(r, x, 0)
+ \sum_{j = 1}^{\num} \paren*{\ov c_0(x) + \sum_{j' \neq j}^{\num} \ov c_{j'}(x)} \ell(r, x, j)}\\
& = \num \paren*{\sum_{j = 0 }^{\num} \ov c_{j}(x)} \bracket*{\sum_{j = 0}^{\num} p_j \ell(r, x, j) - \inf_{r \in \sR} \paren*{\sum_{j = 0}^{\num} p_j \ell(r, x, j) }}
\end{align*}
where we let $p_0 = \frac{\sum_{j = 1}^{\num} \ov c_j(x)}{\num \paren*{\sum_{j = 0 }^{\num} \ov c_{j}(x)}}$ and $p_j = \frac{\ov c_0(x) + \sum_{j' \neq j}^{\num} \ov c_{j'}(x)}{\num \paren*{\sum_{j = 0 }^{\num} \ov c_{j}(x)}}$, $j = \curl*{1, \ldots, \num}$ in the last equality. 
By Lemma~\ref{lemma:aux}, we have
\begin{align*}
& \sum_{j = 0}^{\num} p_j \ell(r, x, j) - \inf_{r \in \sR} \paren*{\sum_{j = 0}^{\num} p_j \ell(r, x, j) }\\
& \geq \Gamma^{-1}\paren*{\sum_{j = 0}^{\num} p_j 1_{\rr(x) \neq j} - \inf_{r \in \sR} \paren*{\sum_{j = 0}^{\num} p_j 1_{\rr(x) \neq j}}} \\
& = \Gamma^{-1}\paren*{\max_{j \in [\num]}p_j - p_{\rr(x)}}\\
& = \Gamma^{-1}\paren*{\frac{\ov c_{\rr(x)}(x) - \ov c_0(x)}{\num \paren*{\sum_{j = 0 }^{\num} \ov c_{j}(x)}}}
\end{align*}
Therefore, we obtain
\begin{align*}
& \E_{y | x}\bracket*{L_{\ell}(h, r, x, y)} - \inf_{h\in \sH,r\in \sR}\E_{y | x}\bracket*{L_{\ell}(h, r, x, y)}\\
& \geq \num \paren*{\sum_{j = 0 }^{\num} \ov c_{j}(x)} \Gamma^{-1}\paren*{\frac{\ov c_{\rr(x)}(x) - \ov c_0(x)}{\num \paren*{\sum_{j = 0 }^{\num} \ov c_{j}(x)}}} \\
& = \num \paren*{\sum_{j = 0 }^{\num} \ov c_{j}(x)} \Gamma^{-1}\paren*{\frac{\E_{y | x}\bracket*{\ldef(h, r, x, y)}-\inf_{h\in \sH,r\in \sR}\E_{y | x}\bracket*{\ldef(h, r, x, y)}}{\num \paren*{\sum_{j = 0 }^{\num} \ov c_{j}(x)}}} \\
& \geq \frac{1}{\beta^{\frac{1}{\alpha}}} \frac{\paren*{\E_{y | x}\bracket*{\ldef(h, r, x, y)}-\inf_{h\in \sH,r\in \sR}\E_{y | x}\bracket*{\ldef(h, r, x, y)}}^{\frac{1}{\alpha}}}{\paren*{\num(\ul + \sum_{j = 1}^{\num}\uc_j)}^{\frac{1}{\alpha} - 1}}
\end{align*}
where we use the fact that $\Gamma(t) = \beta t^{\alpha}$, $\alpha \in (0, 1]$, $\beta > 0$, $\sfL \leq \ul$ and $c_j \leq \uc_j$, $j = \curl*{1, \ldots, \num}$ in the last inequality.

Overall, by taking the expectation of the deferral and surrogate calibration gaps and using Jensen's inequality in each case, we obtain
\begin{equation*}
   \sE_{\ldef}(h, r) - \sE_{\ldef}^*(\sH,\sR) + \sM_{\ldef}(\sH,\sR) \leq 
    \ov \Gamma\paren*{\sE_{L_{\ell}}(h, r) -  \sE_{L_{\ell}}^*(\sH,\sR) + \sM_{L_{\ell}}(\sH,\sR)}.
\end{equation*}
where $\ov \Gamma(t) = \max\curl*{t, \paren*{\num\paren*{\ul + \sum_{j = 1}^{\num}\uc_j}}^{1 - \alpha} \beta\, t^{\alpha}}$.
\end{proof}

\section{Proof of Theorem~\ref{thm:tsr}}
\label{app:tsr}

\TwostageR*
\begin{proof}
Given a hypothesis set $\sR$, a multi-class loss function $\ell$ and a predictor $h$. For any $r \in \sR$, $x \in \sX$ and $y \in \sY$, the conditional error of $L^h_{\ell}$ and $L^h_{\rm{def}}$ can be written as
\begin{equation}
\label{eq:tsr-cond-error}
\begin{aligned}
\E_{y | x}\bracket*{L^h_{\rm{def}}(r, x, y)} & =  \E_{y | x}\bracket*{\sfL(h(x), y)} 1_{\rr(x) = 0} + \sum_{j = 1}^{\num} \E_{y | x}\bracket*{c_j(x,y)} 1_{\rr(x) = j}\\
\E_{y | x}\bracket*{L^h_{\ell}(r, x, y)} & =  \paren*{\sum_{j = 1}^{\num} \E_{y | x}\bracket*{c_j(x,y)}} \ell(r, x, 0)
+ \sum_{j = 1}^{\num} \paren*{\E_{y | x}\bracket*{\sfL(h(x), y)} + \sum_{j' \neq j}^{\num} \E_{y | x}\bracket*{c_{j'}(x,y)}} \ell(r, x, j).
\end{aligned}
\end{equation}
Let $\ov c_0(x) = \inf_{h\in \sH}\E_{y | x}\bracket*{\sfL(h(x), y)}$ and $\ov c_j(x) = \E_{y | x}\bracket*{c(x, y)}$.
Thus, the best-in class conditional error of of $L^h_{\ell}$ and $L^h_{\rm{def}}$ can be expressed as
\begin{equation}
\label{eq:tsr-best-cond-error}
\begin{aligned}
\inf_{r\in \sR}\E_{y | x}\bracket*{L^h_{\rm{def}}(r, x, y)} & =  \min_{j \in [\num]} \ov c_j(x)\\
\inf_{r\in \sR}\E_{y | x}\bracket*{L^h_{\ell}(r, x, y)} & =  \inf_{r \in \sR}\bracket*{\paren*{\sum_{j = 1}^{\num} \ov c_j(x)} \ell(r, x, 0)
+ \sum_{j = 1}^{\num} \paren*{\E_{y | x}\bracket*{\sfL(h(x), y)} + \sum_{j' \neq j}^{\num} \ov c_{j'}(x)} \ell(r, x, j)}
\end{aligned}
\end{equation}
Let $p_0 = \frac{\sum_{j = 1}^{\num} \ov c_j(x)}{\num \paren*{\E_{y | x}\bracket*{\sfL(h(x), y)} + \sum_{j = 1 }^{\num} \ov c_{j}(x)}}$ and $p_j = \frac{\E_{y | x}\bracket*{\sfL(h(x), y)} + \sum_{j' \neq j}^{\num} \ov c_{j'}(x)}{\num \paren*{\E_{y | x}\bracket*{\sfL(h(x), y)} + \sum_{j = 1 }^{\num} \ov c_{j}(x)}}$, $j = \curl*{1, \ldots, \num}$. Then, the calibration gap of $L_{\ell}^h$ can be written as 
\begin{align*}
& \E_{y | x}\bracket*{L^h_{\ell}(r, x, y)} - \inf_{r\in \sR}\E_{y | x}\bracket*{L^h_{\ell}(r, x, y)}\\
& = \num \paren*{\E_{y | x}\bracket*{\sfL(h(x), y)} + \sum_{j = 1 }^{\num} \ov c_{j}(x)}  \bracket*{\sum_{j = 0}^{\num} p_j \ell(r, x, j) - \inf_{r \in \sR} \paren*{\sum_{j = 0}^{\num} p_j \ell(r, x, j) }}
\end{align*}
By Lemma~\ref{lemma:aux}, we have
\begin{align*}
\sum_{j = 0}^{\num} p_j \ell(r, x, j) - \inf_{r \in \sR} \paren*{\sum_{j = 0}^{\num} p_j \ell(r, x, j) } 
& \geq \Gamma^{-1}\paren*{\sum_{j = 0}^{\num} p_j 1_{\rr(x) \neq j} - \inf_{r \in \sR} \paren*{\sum_{j = 0}^{\num} p_j 1_{\rr(x) \neq j}} }\\
& = \Gamma^{-1}\paren*{\max_{j \in [\num]}p_j - p_{\rr(x)}}\\
& = \Gamma^{-1}\paren*{\frac{\ov c_{\rr(x)}(x) - \min_{j \in [\num]} \ov c_j(x)}{\num \paren*{\E_{y | x}\bracket*{\sfL(h(x), y)} + \sum_{j = 1 }^{\num} \ov c_{j}(x)}}}.
\end{align*}
Therefore, we obtain
\begin{align*}
& \E_{y | x}\bracket*{L_{\ell}(r, x, y)} - \inf_{r\in \sR}\E_{y | x}\bracket*{L_{\ell}(r, x, y)}\\
& \geq \num \paren*{\E_{y | x}\bracket*{\sfL(h(x), y)} + \sum_{j = 1 }^{\num} \ov c_{j}(x)} \Gamma^{-1}\paren*{\frac{\ov c_{\rr(x)}(x) - \min_{j \in [\num]} \ov c_j(x)}{\num \paren*{\E_{y | x}\bracket*{\sfL(h(x), y)} + \sum_{j = 1 }^{\num} \ov c_{j}(x)}}}\\
& \geq \frac{1}{\beta^{\frac{1}{\alpha}}} \frac{\paren*{\E_{y | x}\bracket*{L^h_{\rm{def}}(r, x, y)}-\inf_{r\in \sR}\E_{y | x}\bracket*{L^h_{\rm{def}}(r, x, y)}}^{\frac{1}{\alpha}}}{\paren*{\num(\ul + \sum_{j = 1}^{\num}\uc_j)}^{\frac{1}{\alpha} - 1}}
\end{align*}
where we use the fact that $\Gamma(t) = \beta t^{\alpha}$, $\alpha \in (0, 1]$, $\beta > 0$, $\sfL \leq \ul$ and $c_j \leq \uc_j$, $j = \curl*{1, \ldots, \num}$ in the last inequality.
Taking the expectation on both sides and using Jensen's inequality, we obtain
\begin{equation*}
   \sE_{L^h_{\rm{def}}}(r) - \sE_{L^h_{\rm{def}}}^*(\sR) + \sM_{L^h_{\rm{def}}}(\sR) \leq 
    \ov \Gamma\paren*{\sE_{L^h_{\ell}}(r) -  \sE_{L^h_{\ell}}^*(\sR) + \sM_{L^h_{\ell}}(\sR)}.
\end{equation*}
where $\ov \Gamma(t) = \paren*{\num\paren*{\ul + \sum_{j = 1}^{\num}\uc_j}}^{1 - \alpha} \beta\, t^{\alpha}$.
\end{proof}

\section{Proof of Theorem~\ref{thm:tshr}}
\label{app:tshr}

\TwostageHR*
\begin{proof}
The conditional error of the deferral loss can be expressed as
\begin{equation*}
\begin{aligned}
\E_{y | x}\bracket*{\ldef(h, r, x, y)} =
\E_{y | x}\bracket*{\sfL(h(x), y)} 1_{\rr(x) = 0} + \sum_{j = 1}^{\num} \E_{y | x}\bracket*{c_j(x,y)} 1_{\rr(x) = j}.
\end{aligned}
\end{equation*}
Let $\ov c_0(x) = \inf_{h\in \sH}\E_{y | x}\bracket*{\sfL(h(x), y)}$ and $\ov c_j(x) = \E_{y | x}\bracket*{c(x, y)}$.
Thus, the best-in class conditional error of the deferral loss can be expressed as
\begin{equation*}
\inf_{h\in \sH, r\in \sR}\E_{y | x}\bracket*{\ldef(h, r, x, y)} =  \min_{j \in [\num]} \ov c_j(x).
\end{equation*}
Thus, by introducing the term $\min\curl*{\E_{y | x}\bracket*{\sfL(h(x), y)}, \min_{j = 1}^{\num} \ov c_j(x)}$ and subsequently subtracting it after rearranging, the conditional regret of the deferral loss $\ldef$ can be written as follows
\begin{equation}
\label{eq:tshr-cond-reg-def}
\begin{aligned}
& \E_{y | x}\bracket*{\ldef(h, r, x, y)} - \inf_{h\in \sH, r\in \sR}\E_{y | x}\bracket*{\ldef(h, r, x, y)}\\
& = \E_{y | x}\bracket*{\sfL(h(x), y)} 1_{\rr(x) = 0} + \sum_{j = 1}^{\num} \E_{y | x}\bracket*{c_j(x,y)} 1_{\rr(x) = j} - \min_{j \in [\num]} \ov c_j(x)\\ 
& =  \E_{y | x}\bracket*{\sfL(h(x), y)} 1_{\rr(x) = 0} + \sum_{j = 1}^{\num} \E_{y | x}\bracket*{c_j(x,y)} 1_{\rr(x) = j} - \min_{j = 1}^{\num} \ov c_j(x) + \paren*{\min_{j = 1}^{\num} \ov c_j(x) - \min_{j \in [\num]} \ov c_j(x)}.
\end{aligned}
\end{equation}
Note that by the property of the minimum,  the second term can be upper bounded as 
\begin{align*}
\min_{j = 1}^{\num} \ov c_j(x) - \min_{j \in [\num]} \ov c_j(x) \leq \E_{y | x}\bracket*{\sfL(h(x), y)} - \inf_{h \in \sH}\E_{y | x}\bracket*{\sfL(h(x), y)}.
\end{align*}
Next, we will upper bound the first term. Note that the conditional error and the best-in class conditional error of $L^h_{\ell}$ can be expressed as
\begin{equation}
\label{eq:tshr-cond-error-sur}
\begin{aligned}
L^h_{\ell}(r, x, y) & =  \paren*{\sum_{j = 1}^{\num} \ov c_j(x)} \ell(r, x, 0)
+ \sum_{j = 1}^{\num} \paren*{\E_{y | x}\bracket*{\sfL(h(x), y)} + \sum_{j' \neq j}^{\num} \ov c_{j'}(x) } \ell(r, x, j)\\
\inf_{r\in \sR}\E_{y | x}\bracket*{L^h_{\ell}(r, x, y)} & = \inf_{r \in \sR}\bracket*{\paren*{\sum_{j = 1}^{\num} \ov c_j(x)} \ell(r, x, 0)
+ \sum_{j = 1}^{\num} \paren*{\E_{y | x}\bracket*{\sfL(h(x), y)} + \sum_{j' \neq j}^{\num} \ov c_{j'}(x)} \ell(r, x, j)}
\end{aligned}
\end{equation}
Let $p_0 = \frac{\sum_{j = 1}^{\num} \ov c_j(x)}{\num \paren*{\E_{y | x}\bracket*{\sfL(h(x), y)} + \sum_{j = 1 }^{\num} \ov c_{j}(x)}}$ and $p_j = \frac{\E_{y | x}\bracket*{\sfL(h(x), y)} + \sum_{j' \neq j}^{\num} \ov c_{j'}(x)}{\num \paren*{\E_{y | x}\bracket*{\sfL(h(x), y)} + \sum_{j = 1 }^{\num} \ov c_{j}(x)}}$, $j = \curl*{1, \ldots, \num}$. Then, the first term can be rewritten as 
\begin{align*}
& \E_{y | x}\bracket*{\sfL(h(x), y)} 1_{\rr(x) = 0} + \sum_{j = 1}^{\num} \E_{y | x}\bracket*{c_j(x,y)} 1_{\rr(x) = j} - \min_{j = 1}^{\num} \ov c_j(x) \\
& = \num \paren*{\E_{y | x}\bracket*{\sfL(h(x), y)} + \sum_{j = 1 }^{\num} \ov c_{j}(x)}  \bracket*{\sum_{j = 0}^{\num} p_j 1_{\rr(x) \neq 0} - \inf_{r \in \sR} \paren*{\sum_{j = 0}^{\num} p_j 1_{\rr(x) \neq j}}}.
\end{align*}
By Lemma~\ref{lemma:aux}, we have
\begin{align*}
\sum_{j = 0}^{\num} p_j 1_{\rr(x) \neq j} - \inf_{r \in \sR} \paren*{\sum_{j = 0}^{\num} p_j 1_{\rr(x) \neq j}} 
& \leq \Gamma\paren*{\sum_{j = 0}^{\num} p_j \ell(r, x, j) - \inf_{r \in \sR} \paren*{\sum_{j = 0}^{\num} p_j \ell(r, x, j) }}\\
& = \Gamma\paren*{\frac{L^h_{\ell}(r, x, y) - \inf_{r\in \sR}\E_{y | x}\bracket*{L^h_{\ell}(r, x, y)}}{\num \paren*{\E_{y | x}\bracket*{\sfL(h(x), y)} + \sum_{j = 1 }^{\num} \ov c_{j}(x)}}}.
\end{align*}
Therefore, the first term can be upper bounded as
\begin{align*}
&  \E_{y | x}\bracket*{\sfL(h(x), y)} 1_{\rr(x) = 0} + \sum_{j = 1}^{\num} \E_{y | x}\bracket*{c_j(x,y)} 1_{\rr(x) = j} - \min_{j = 1}^{\num} \ov c_j(x)\\
& = \num \paren*{\E_{y | x}\bracket*{\sfL(h(x), y)} + \sum_{j = 1 }^{\num} \ov c_{j}(x)}  \bracket*{\sum_{j = 0}^{\num} p_j 1_{\rr(x) \neq 0} - \inf_{r \in \sR} \paren*{\sum_{j = 0}^{\num} p_j 1_{\rr(x) \neq j}}}\\
& \leq \num \paren*{\E_{y | x}\bracket*{\sfL(h(x), y)} + \sum_{j = 1 }^{\num} \ov c_{j}(x)} \Gamma\paren*{\frac{L^h_{\ell}(r, x, y) - \inf_{r\in \sR}\E_{y | x}\bracket*{L^h_{\ell}(r, x, y)}}{\num \paren*{\E_{y | x}\bracket*{\sfL(h(x), y)} + \sum_{j = 1 }^{\num} \ov c_{j}(x)}}}\\
& \leq \paren*{\num\paren*{\ul + \sum_{j = 1}^{\num}\uc_j}}^{1 - \alpha} \beta\, \paren*{L^h_{\ell}(r, x, y) - \inf_{r\in \sR}\E_{y | x}\bracket*{L^h_{\ell}(r, x, y)}}^{\alpha}
\end{align*}
where we use the fact that $\Gamma(t) = \beta t^{\alpha}$, $\alpha \in (0, 1]$, $\beta > 0$, $\sfL \leq \ul$ and $c_j \leq \uc_j$, $j = \curl*{1, \ldots, \num}$ in the last inequality.
After upper bounding the first term and the second term in \eqref{eq:tshr-cond-reg-def} as above, taking the expectation on both sides and using Jensen's inequality, we obtain
\begin{align*}
   \sE_{\ldef}(h, r) - \sE_{\ldef}^*(\sH,\sR) + \sM_{\ldef}(\sH,\sR) & \leq \sE_{\sfL}(h) - \sE_{\sfL}(\sH) + \sM_{\sfL}(\sH)\\
   & \quad + \ov \Gamma\paren*{\sE_{L^h_{\ell}}(r) -  \sE_{L^h_{\ell}}^*(\sR) + \sM_{L^h_{\ell}}(\sR)},
\end{align*}
where $\ov \Gamma(t) = \paren*{\num\paren*{\ul + \sum_{j = 1}^{\num}\uc_j}}^{1 - \alpha} \beta\, t^{\alpha}$.
\end{proof}

\section{Common margin-based losses and corresponding deferral surrogate losses}
\label{app:sur-binary}
\begin{table}[h]
\caption{Common margin-based losses and corresponding deferral surrogate losses.}
  \label{tab:sur-binary}
  \centering
  \resizebox{\columnwidth}{!}{
  \begin{tabular}{@{\hspace{0cm}}lll@{\hspace{0cm}}}
    \toprule
      Name & $\Phi(u)$ & Deferral surrogate loss $\ell_{\Phi}$\\
    \midrule
     Exponential & $\Phi_{\rm{exp}}(u) = e^{-u}$ & $\sfL(h(x), y) e^{r(x)} + c(x,y) e^{-r(x)}$    \\
     Logistic & $\Phi_{\rm{log}}(u) = \log\paren*{1 + e^{-u}}$ & $\sfL(h(x), y) \log\paren*{1 + e^{r(x)}} + c(x,y) \log\paren*{1 + e^{-r(x)}}$ \\
     Quadratic & $\Phi_{\rm{quad}}(u) = \max\curl*{1 - u, 0}^2$ & $\sfL(h(x), y) \Phi_{\rm{quad}}(-r(x)) + c(x,y) \Phi_{\rm{quad}}(r(x))$ \\
     Hinge & $\Phi_{\rm{hinge}}(u) = \max\curl*{1 - u, 0}$ & $\sfL(h(x), y) \Phi_{\rm{hinge}}(-r(x)) + c(x,y) \Phi_{\rm{hinge}}(r(x))$ \\
     Sigmoid & $\Phi_{\rm{sig}}(u) = 1 - \tanh(k u), k > 0$ & $\sfL(h(x), y) \Phi_{\rm{sig}}(-r(x)) + c(x,y) \Phi_{\rm{sig}}(r(x))$  \\
     $\rho$-Margin & $\Phi_{\rho}(u) = \min\curl*{1,  \max\curl*{0, 1-\frac{u}{\rho}}}, \rho>0$ & $\sfL(h(x), y) \Phi_{\rho}(-r(x)) + c(x,y) \Phi_{\rho}(r(x))$ \\ 
    \bottomrule
  \end{tabular}
  }
\end{table}

\end{document}